\documentclass[letterpaper]{article} 
\usepackage{aaai25}  
\usepackage{times}  
\usepackage{helvet}  
\usepackage{courier}  
\usepackage[hyphens]{url}  
\usepackage{graphicx} 
\usepackage{natbib}  
\usepackage{caption} 
\usepackage{algorithm}
\usepackage{algorithmic}
\usepackage{subcaption}
\usepackage{amssymb}
\usepackage{amsmath}
\usepackage{amsthm}
\usepackage{booktabs}
\usepackage{multirow}

\usepackage{newfloat}
\usepackage{listings}
\DeclareCaptionStyle{ruled}{labelfont=normalfont,labelsep=colon,strut=off} 
\floatstyle{ruled}
\newfloat{listing}{tb}{lst}{}
\floatname{listing}{Listing}
\title{Self-Explainable Graph Transformer for Link Sign Prediction}
\author{
     Lu Li, 
    Jiale Liu\, 
    Xingyu Ji, 
    Maojun Wang$^{\ast}$, 
    Zeyu Zhang\thanks{Corresponding authors:  Zeyu Zhang and Maojun Wang}
}
\affiliations{
    National Key Laboratory of Crop Genetic Improvement, Hubei Hongshan Laboratory, Huazhong Agricultural University\\

    lu123@webmail.hzau.edu.cn, 1723383466@webmail.hzau.edu.cn, erican\_j@webmail.hzau.edu.cn, 
   mjwang@mail.hzau.edu.cn, 
  zhangzeyu@mail.hzau.edu.cn\\
}

\usepackage{bibentry}
\newtheorem{theorem}{Theorem}

\theoremstyle{definition}
\newtheorem{definition}{Definition}
\begin{document}

\maketitle

\begin{abstract}
Signed Graph Neural Networks (SGNNs) have been shown to be effective in analyzing complex patterns in real-world situations where positive and negative links coexist. However, SGNN models suffer from poor explainability, which limit their adoptions in critical scenarios that require understanding the rationale behind predictions. To the best of our knowledge, there is currently no research work on the explainability of the SGNN models. Our goal is to address the explainability of decision-making for the downstream task of link sign prediction specific to signed graph neural networks. Since post-hoc explanations are not derived directly from the models, they may be biased and misrepresent the true explanations. Therefore, in this paper we introduce a \underline{S}elf-\underline{E}xplainable \underline{S}igned \underline{G}raph trans\underline{former} (SE-SGformer) framework, which can not only outputs explainable information while ensuring high prediction accuracy. Specifically, we propose a new Transformer architecture for signed graphs and theoretically demonstrate that using positional encoding based on signed random walks has greater expressive power than current SGNN methods and other positional encoding graph Transformer-based approaches. We construct a novel explainable decision process by discovering the $K$-nearest (farthest) positive (negative) neighbors of a node to replace the neural network-based decoder for predicting edge signs. These $K$ positive (negative) neighbors represent crucial information about the formation of positive (negative) edges between nodes and thus can serve as important explanatory information in the decision-making process. We conducted experiments on several real-world datasets to validate the effectiveness of SE-SGformer, which outperforms the state-of-the-art methods by improving 2.2\% prediction accuracy and 73.1\% explainablity accuracy in the best-case scenario. The code is available at: \url{https://github.com/liule66/SE-SGformer}.
\end{abstract}

%
\section{Introduction}
Signed Graph Neural Networks (SGNNs) are widely used for learning representations of signed graphs, as shown in Figure \ref{fig:intro}. Despite recent years have witnessed a growing interest in SGNNs with link sign prediction as the focal task \cite{derr2018signed,shu2021sgcl,zhang2023rsgnn,ni2024enhancing}, no existing study on SGNNs has addressed the issue of explainability, which hinders their adoption in crucial domains. For example, in financial networks, being able to interpret why certain transactions are flagged as suspicious can improve fraud detection and prevention. Therefore, understanding why certain relationships are classified as positive or negative can help SGNNs be more widely applied.

\begin{figure}
    \centering
    \includegraphics[width=\columnwidth]{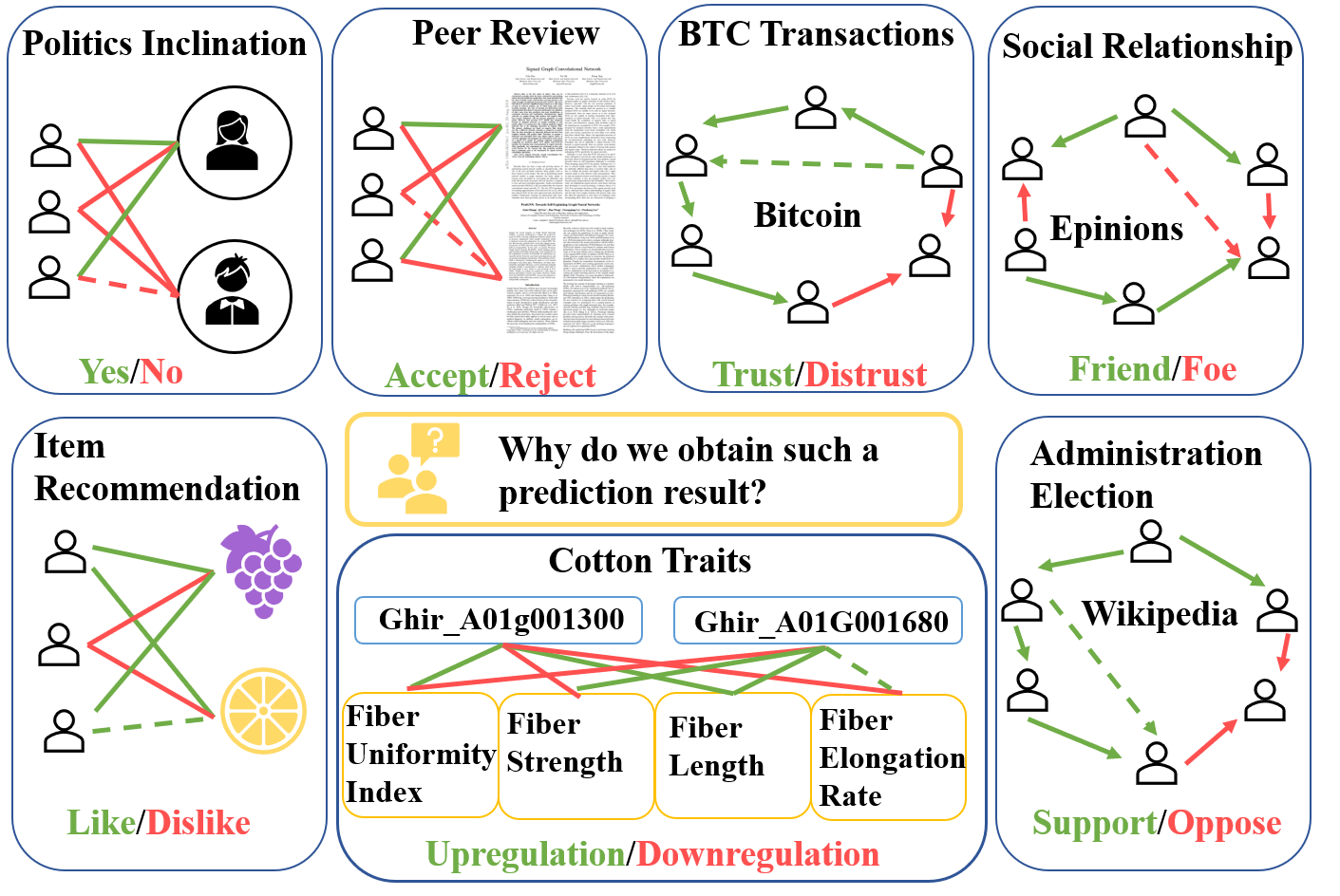}
    \caption{An illustration of signed graphs in real world. The dashed lines represent the edges to be predicted, and black and red lines represent positive and negative edges, resp.}
    \label{fig:intro}
\end{figure}

Research on GNN explainability falls into two main categories: post-hoc explanations \cite{zhang2024trustworthy,yuan2022explainability} and self-explainable approaches \cite{deng2024self,seo2024interpretable}.
Post-hoc explanation methods like GNNExplainer \cite{ying2019gnnexplainer} and PGExplainer \cite{luo2020parameterized} offer interpretations for trained GNN models, but these explanations may be biased and not truly reflective of the model. As a result, current research is shifting toward self-explaining methods, where models generate explanations alongside predictions. For example, SE-GNN \cite{dai2021towards} uses $K$-nearest labeled nodes for explainable node classification, while ProtGNN \cite{zhang2022protgnn} integrates prototype learning to enhance interpretability. While effective, these methods are primarily designed for unsigned graphs with graph or node classification tasks, making them unsuitable for SGNNs focusing on link sign prediction. This highlights the need for a new explainable framework for signed graph representation learning that can offer both predictions and explanations.

One potential approach for obtaining self-explanations in link sign prediction is to identify explainable $K$-nearest positive neighbors (and $K$-farthest negative neighbors) for each node. If the closeness between the two nodes of an edge $e_{ij}$ is similar to the closeness between node $v_{i}$ and its $K$-nearest positive (or $K$-farthest negative) neighbors, the predicted sign of the edge is positive (or negative). The identified $K$-nearest positive (or $K$-farthest negative) neighbors then serve as the explanatory information. This approach simultaneously provides a prediction result and an explanation for that prediction.

The key to identifying the $K$-nearest (farthest) positive (negative) neighbors for each node lies in learning proper node representations and selecting an adequate number of positive and negative neighbors for any given node. The challenges can be divided into the two aspects:

\textit{Challenge 1} (Encoding part): How to improve the quality of node representations. To prevent overfitting, current SGNN models limit networks to three layers or fewer, restricting their ability to capture multi-hop information \cite{derr2018signed}. Meanwhile, GCN-based SGNN frameworks cannot learn proper representations from unbalanced cycles \cite{zhang2023rsgnn}.

\textit{Challenge 2} (Explanation part): How to find a sufficient number of negative neighbors for nodes. As is shown in Figure \ref{fig:neg-degree}, a significant proportion of nodes have few or even no negative neighbors. For example, in the Bitcoin-Alpha dataset, over 80\% of the nodes have no negative neighbors.


\begin{figure}
    \centering
    \includegraphics[width=0.8\linewidth]{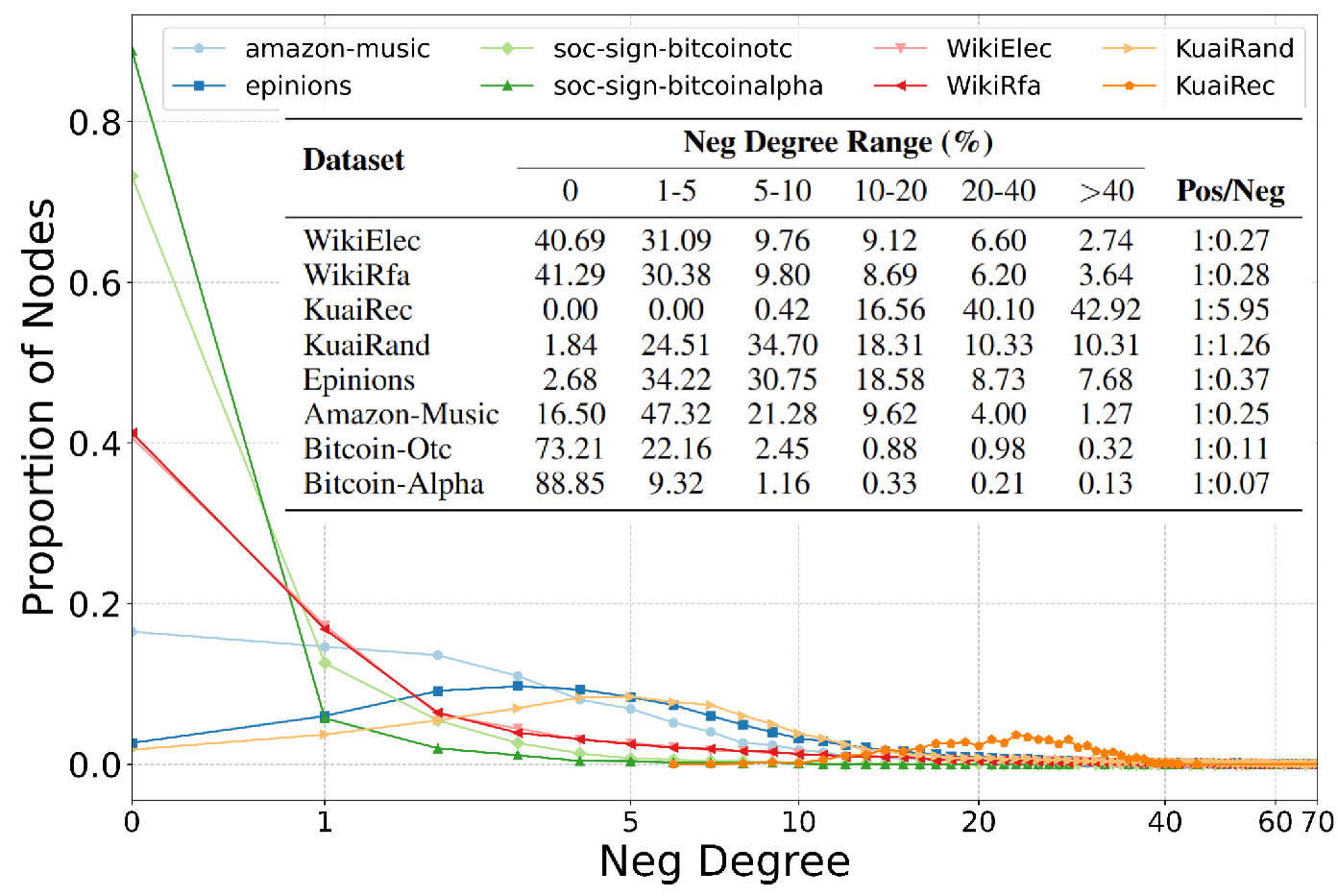}
    \caption{Proportion of node negative degrees and Pos/Neg Ratios across datasets (\%)}
    \label{fig:neg-degree}
\end{figure}


For \underline{Challenge 1}, we design a novel graph Transformer architecture specifically for signed graphs (see \textbf{Encoding Module} section). Specifically, we use signed random walk encoding to help capture multi-hop neighbor information. We theoretically demonstrate that this encoding method has a stronger representation power compared to current SGNNs and other common encoding methods, such as the shortest path encoding.


For \underline{Challenge 2}, we employ signed diffusion matrix based on Signed Random Walk with Restart (SRWR) algorithm \cite{jung2016personalized} (see \textbf{Explainable Prediction Module} section) to uncover potential negative relationships among nodes. 

Overall, we propose a novel Self-Explainable Signed Graph Transformer (SE-SGformer) framework which utilizes the Transformer architecture as an encoder for the signed graph, predicting edge signs by identifying the $K$-nearest (farthest) positive (negative) neighbors and providing corresponding explanatory information. We conduct experiments on eight real-world datasets to validate the effectiveness of SE-SGformer, which surpasses state-of-the-art methods by achieving a 2.2\% increase in prediction accuracy and a 73.1\% improvement in explainability accuracy under optimal conditions. Related work can be found in Appendix. Our contributions are summarized as follows:


\begin{itemize}
    \item We first propose a self-explainable model SE-SGformer for signed graphs, which is specifically designed for the link sign prediction task and can provide explanations along with the predictions.
    \item We propose a novel graph Transformer-based representation learning model for signed graphs and theoretically prove that our random walk encoding is  more powerful than the shortest path encoding and our Transformer architecture with random walk encoding is more powerful than SGCN.
    \item We conduct extensive experiments on several real-world datasets
    and prove that SE-SGformer achieve performance comparable to state-of-the-art models and achieve good explanatory accuracy.
\end{itemize}

\section{Problem Definition}

Let $\mathcal{G}=\left(\mathcal{V}, \mathcal{E}^{+}, \mathcal{E}^{-}\right)$ be a signed network, where $\mathcal{V}=\left\{v_{1}, v_{2}, \ldots v_{|\mathcal{V}|}\right\}$ represents the set of $|\mathcal{V}|$ nodes, while $\mathcal{E}^{+}\subset\mathcal{V}\times\mathcal{V}$ and $\mathcal{E}^{-}\subset\mathcal{V}\times\mathcal{V}$ denote the sets of positive and negative links, respectively. For each edge $e_{ij} \in \mathcal{E}^+ \cup \mathcal{E}^-$ connecting two nodes $v_i$ and $v_j$, the edge can be either positive or negative, but not both, implying $\mathcal{E}^+ \cap \mathcal{E}^- = \emptyset$. The graph structure can be described by an adjacency matrix
$A \in\mathbb{R}^{|\mathcal{V}|\times |\mathcal{V}|}$, where $A_{ij} = 1$ means that there exists a positive link from $v_{i}$ to $v_{j}$ , $A_{ij} = -1$ denotes a negative link, and $A_{ij} = 0$ otherwise (meaning that there is no link from $v_{i}$ to $v_{j}$). Note that real-world signed graph datasets typically do not provide node features. Therefore, there is no feature vector $x_i$ associated with each node $v_i$.

The goal of a SGNN is to learn an embedding function \textit{} $f_{\theta}: \mathcal{V} \rightarrow \mathcal{Z}$, which maps the nodes of a signed graph to a latent vector space $\mathcal{Z}$. In this space, $f_{\theta}(v_i)$ and $f_{\theta}(v_j)$ are close if $e_{ij} \in \mathcal{E}^+$ and distant if $e_{ij} \in \mathcal{E}^-$. Furthermore, we adopt \textit{ link sign prediction} as the downstream task for SGNN, following the mainstream studies. This task aims to infer the sign of a link given the nodes $v_i$ and $v_j$. The link sign prediction can be explained as follows: (1) for a node pair $(v_i, v_j)$ connected by a directed positive edge, taking node $v_i$ as an example, $v_i$'s $K$-nearest positive neighbors have a higher similarity score with respect to $v_j$; (2) for a node pair $(v_i, v_j)$ connected by a directed negative edge, taking node $v_i$ as an example, $v_i$'s $K$-farthest negative neighbors have a higher similarity score with respect to $v_j$. With the aforementioned notations, we formualte our Self-explainable link sign prediction problem as:

Given a signed graph $\mathcal{G}$, the task is to learn the SGNN parameters $\theta$ while simultaneously \textit{predicting link signs} and \textit{generating explanations} by identifying the set of $K$-nearest (farthest) positive (negative) neighbors.

\begin{figure*}
    \centering
    \includegraphics[width=0.94\linewidth]{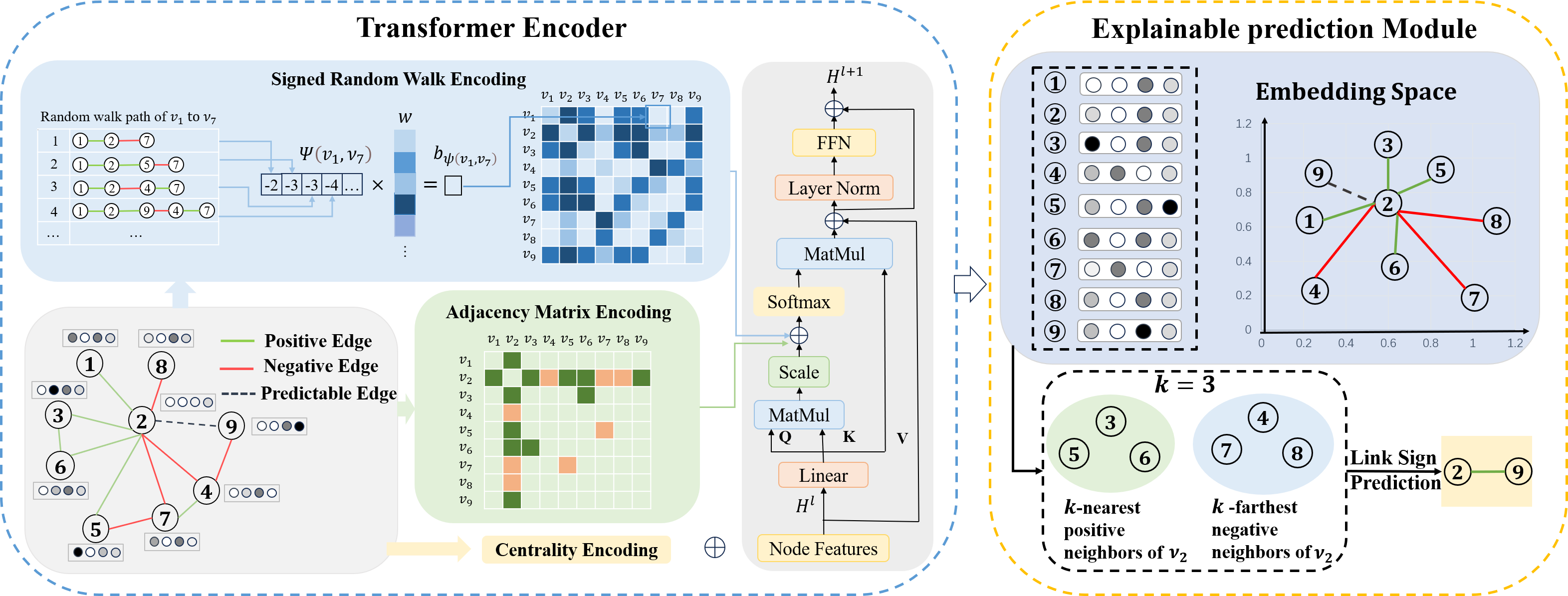}
    \caption{The overall architecture of SE-SGformer. Firstly, the Transformer encoder, equipped with centrality encoding, adjacency matrix encoding, and random walk encoding, obtain node embeddings. Next, in the explainable prediction module, the $K$-nearest positive (and $K$-farthest negative) neighbors of the nodes are identified and used to predict the unknown link sign.}
    \label{fig:framework}
\end{figure*}

\section{Proposed Method}
In this section, we introduce the details of the proposed framework SE-SGformer. The basic idea of SE-SGformer is to first use a Transformer to encode the signed graph. Then, for an edge with an unknown sign from node $v_i$ to node $v_j$, we identify the $K$-nearest (farthest) positive (negative) neighbors of $v_i$. The sign of the edge $e_{ij}$ is determined based on the similarity between node $v_j$ and the positive or negative neighbors of node $v_i$, with explanatory information provided simultaneously. The overview of our proposed SE-SGformer is shown in Figure \ref{fig:framework}. It is mainly divided into two parts: encoding module and explainable prediction module.


\subsection{Encoding Module} 
The classical Transformer architecture \cite{vaswani2017attention} is composed of self-attention modules and feed-forward neural networks. In the self-attention module, the input $H$ $\in\mathbb{R}^{n\times d}$ are projected to the corresponding
query $Q$, key $K$, and value $V$. The self-attention is then calculated as:
\begin{equation}\footnotesize
    Q=HW_Q,\quad K=HW_K,\quad V=HW_V,
\end{equation}
\begin{equation}\footnotesize
    \tilde{A} = \frac{{QK}^T}{\sqrt{d_K}}, \mathrm{Attn}(H)=\mathrm{softmax}(\tilde{A})V
\end{equation}

where $\tilde{A}$ denotes the matrix capturing the similarity between queries and keys. $W_{Q} \in \mathbb{R}^{d \times d_{K}}, W_{K} \in \mathbb{R}^{d \times d_{K}}, W_{V} \in \mathbb{R}^{d \times d_{V}}$
denote the
projected matrices of query, key and value, respectively.

For graph representation, incorporating structural information of graphs into Transformer models is crucial. Therefore, We introduce three new encodings. The specific details are as follows:

\textit{Centrality Encoding}. Degree centrality measures the influence of a node in a graph, where higher degrees indicate potentially greater influence. However, these insights are often overlooked in attention computations. Therefore, we design a centrality encoding tailored for signed graphs to reflect the importance of each node. To be specific, we assign two real-valued embedding vectors to each node based on its positive degree and negative degree. We directly add the centrality encoding of each node to its original features $x_{i}$ to form the new node embeddings:

\begin{equation}\footnotesize
    h_i^{(0)}=x_i+c_{\deg^-(v_i)}^-+c_{\deg^+(v_i)}^+,
    \label{eq:centrality}
\end{equation}

where $c^-, c^+\in\mathbb{R}^d$ are learnable embedding vectors determined by the negative degree $\operatorname{deg}^-(v_i)$ and positive degree $\operatorname{deg}^+(v_i)$, respectively. 

\textit{Adjacency Matrix Encoding.} Unlike the Transformer's approach to sequential data, where specific positional encodings are added to each word to indicate its position in a sentence, nodes in a graph are not arranged in a sequence. To encode the structural information of a graph in the model, we introduce the adjacency matrix encoding. $A$ is the adjacency matrix of the graph, characterizing their direct positive and negative neighbors. After normalizing the adjacency matrix, we obtain the adjacency matrix encoding $\hat{A}$ which serves as a bias term in the self-attention module. Denote $\tilde{A}_{ij}$ as the $(i, j)$-element of the Query-Key product matrix $\tilde{A}$, we have:
\begin{equation}\footnotesize
    \tilde{A}_{ij}=\frac{(h_iW_Q)(h_jW_K)^T}{\sqrt{d}}+\hat{A}_{ij},
    \label{eq:adjacency}
\end{equation}

\textit{Signed Random Walk Encoding.} Adjacency matrix encoding can only capture the relationships between nodes and their one-hop neighbors, but it cannot represent the relationships between nodes and their multihop neighbors. Previous works introduce spatial encoding to capture the relationships between a node and its multi-hop neighbors, such as the shortest path encoding \cite{ying2021Transformers}. However, the shortest-path encoding approach only considers a single path, without taking alternative routes into account. Inspired by \cite{yeh2023random}, we introduce the \textit{signed random walk encoding} to exploit the relative position between the nodes and their high-order neighbors by multiple paths from the signed random walk. Generally, we perform multiple random walks on a signed graph $\mathcal{G}$, obtaining multiple random walk sequences $Q=\{q_{i}\}_{i=1}^{l}$ of length $l$, where $q_{i} \in \mathcal{V}$. The random walk sequence in the graph starts from an initial node $q_{i}$, and the next node $q_{i+1}$ is randomly sampled from its neighbors $\mathcal{N}(q_i)$. Additionally, to discover long-range patterns, we use a non-backtracking approach where the previous node is excluded when sampling the next node, unless it is the only neighbor available, ensuring that each node's predecessor and successor are distinct. We define a function $\sigma: e_{ij} \rightarrow \{ 1, -1 \}$. Given a random walk sequence $Q$, the signed random walk distance $\psi(v_{i}, v_{j})$ is defined as follows:
\begin{equation}\footnotesize
    \psi(v_i,v_j)=\\\begin{cases}\Big\{(\prod_{u=n}^{m-1} \sigma (q_u, q_{u+1}))\cdot |m-n|\Big|   \\
    q_m=v_i\wedge q_n=v_j\Big\},&\text{if} v_i,v_j\in Q\\m+1,&\text{otherwise}\end{cases}
\end{equation}

where $m$ denotes the maximum path length between two nodes. If the path length between two nodes is $m+1$, indicating that the nodes are unreachable from each other. So after performing $r$ random walks on graph G, each node pair can obtain $r$ random walk distances and the random walk encoding is defined as follows:
\begin{equation}\footnotesize
    b_{\psi(v_i,v_j)} = w_{r}\cdot\frac{1}{\psi^r(v_i,v_j)}
    \label{eq:random}
\end{equation}
where $w_{r}$ is a learnable parameter that represents the weight of the result obtained from the $r$-th random walk.
Then we modify the $(i, j)$-element of $\tilde{A}$ further with the signed random walk encoding $b$ as:
\begin{equation}\footnotesize
    \tilde{A}_{ij}=\frac{(h_iW_Q)(h_jW_K)^T}{\sqrt{d}}+\hat{A}_{ij}+ b_{\psi(v_i,v_j)}
    \label{eq:attention}
\end{equation}  

Next, following the proof idea of \cite{yeh2023random}, we theoretically analyzed the expressive power of our signed random walk encoding compared to shortest path encoding. Also, we  verified this through experiments and the results can be found in the Appendix. The shortest path encoding for signed graphs calculates the distance between two nodes and, using balance theory, assigns a positive or negative relationship to the path.  
\begin{definition}[Signed graph isomorphism]
Two signed graphs $G_1$ and $G_2$ are isomorphic, if there exists a bijection $\phi$ : $\mathcal{V}_{\mathcal{G}_1}\to\mathcal{V}_{\mathcal{G}_2}$,  for every pair of nodes $v_1, v_2 \in V_{\mathcal{G}_1}$, $e_{ij} \in {\mathcal{E}_1}$, if and only if $e_{\phi(v_i),\phi(v_j)}\in\mathcal{E}_2$ and $\sigma(e_{ij})=\sigma(e_{\phi(v_i),\phi(v_j)}) $.
\end{definition}   

\begin{theorem}
\label{the:1}
With a sufficient number of random walks, signed random walk encoding is more expressive than that based on a fixed shortest path for signed graph.
\end{theorem}

\begin{proof}
We demonstrate our proof with two cases: 1) If two signed graphs are isomorphic under signed random walk encoding, they are also isomorphic under shortest path encoding. 2) In some cases, distinct signed graphs identified as non-isomorphic by signed random walk encoding are indistinguishable by shortest path encoding.

In case 1, given a node pair ($v_{i}$, $v_{j}$), the signed random walk distance varies depending on different random walk samples. Let  $\mathcal{S}_{i,j}=\{s_1,...,s_r\}$ denote the set of $r$ possible distances between $v_{i}$ and $v_{j}$ sampled by different signed random walk. With a sufficient number of random walks, the shortest distance with a positive or negative sign can be explored in $S_{i, j}$. If two graphs $G_1$ and $G_2$ are identified as isomorphic by signed random walk encoding, then every node pair ($v_i^1$, $v_j^1$) in $G_1$ can find a corresponding node pair ($v_i^2$, $v_j^2$) in $G_2$ and $\sigma(e_{v_i^1,v_j^1})=\sigma(e_{v_i^2,v_j^2})$ so corresponding $S_{i,j}^1$, $S_{i,j}^2$ are identical. Therefore, the shortest distances of the two nodes are also the same, $G_1$ and $G_2$ can also be identified as isomorphic by the shortest path encoding.

As illustrated in case 2 (Figure \ref{fig:non-isomorphic}), the structures of the two graphs are different. However, the shortest path encoding cannot identify these two graphs as non-isomorphic. In the left graph, the shortest path distances from node $v_1$ and node $v_4$ to the other nodes are $\{-1, -1, 1, -2, -2\}$ and the shortest path distances from the remaining nodes to all other nodes are $\{-1, 1, 1, 2, -2\}$. For $v_1$ and $v_4$ in the right graph, the  shortest path distances to the other nodes are also $\{-1, -1, 1, -2, -2\}$ and the shortest path distances to the other nodes are also $\{-1, 1, 1, 2, -2\}$ for the remaining nodes. Therefore, the shortest path encoding views them as isomorphic. In contrast, signed random walk encoding regards two graphs as non-isomorphic. For example, one can revisit the node in 3 steps in the left graph, while it is impossible on the right graph. Therefore, signed random walk encoding is more powerful than that based on a fixed shortest path. 
\end{proof}

\begin{figure}
    \centering
    \includegraphics[width=0.8\linewidth]{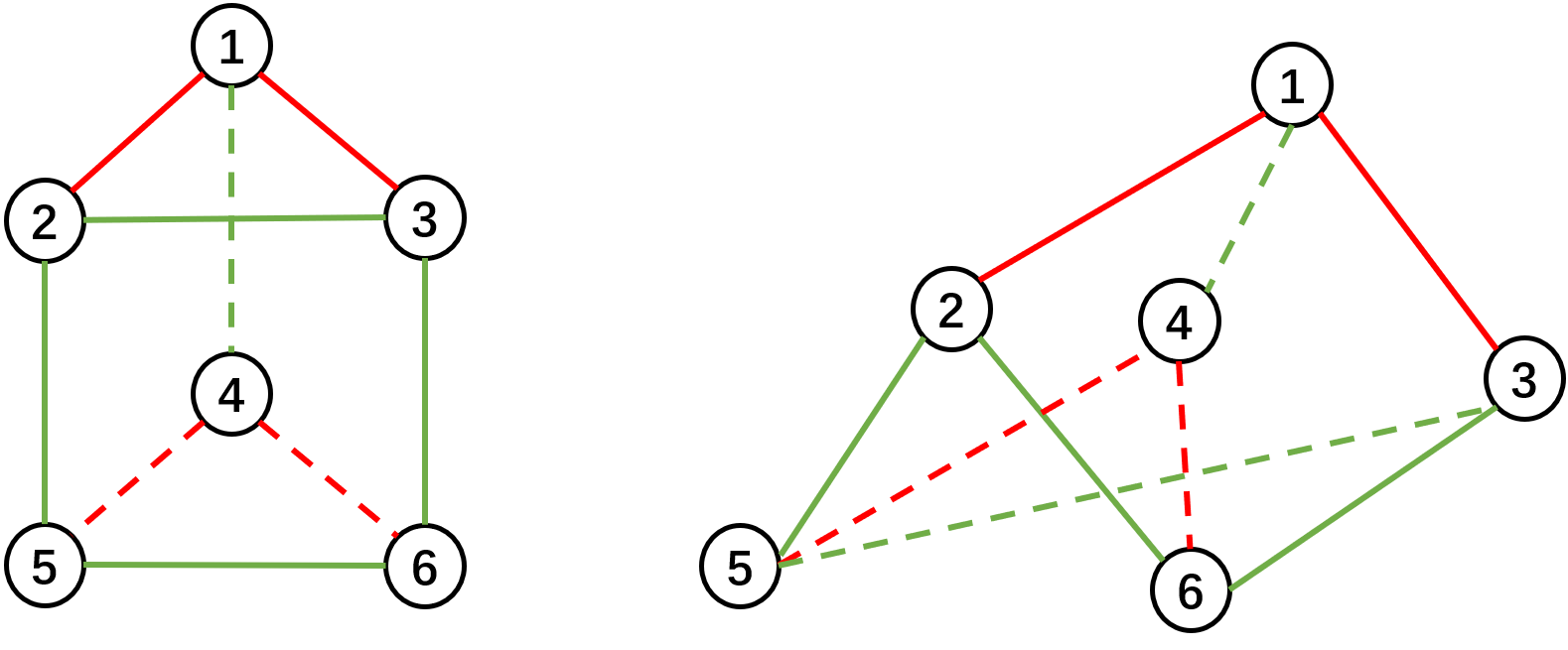}
    \caption{Encoding based on the shortest path cannot map
these two graphs to different embeddings, while encoding
based on the signed random walk can map them to different embeddings.}
    \label{fig:non-isomorphic}
\end{figure}

For the issues present in Challenge 1, we further analyzed the expressive power of our graph transformer architecture based on signed random walk encoding and SGCN-based SGNN models. Our conclusion is as follow:
\begin{theorem}(Proof in Appendix)
\label{the:2}
With a sufficient number of random walk length, the graph transformer architecture based on signed random walk encoding is more expressive than SGCN.
\end{theorem}

According to the paper \cite{zhang2023rsgnn}, the positive or negative relationship between two nodes in an unbalanced cycle is contradictory (
One path is positive, while the other is negative). Therefore, it is impossible to infer their relationship from the local structure. Signed random walk encoding combines information from multiple weighted paths to judge the relationship between nodes. This allows us to help the model obtain reasonable association information between two points within an unbalanced cycle.

\subsubsection{Transformer Layer.} 
In the self-attention layer, each attention head calculation formula is as follow:
\begin{equation}\footnotesize
    \mathrm{Attn}(h^{(l-1)})=\mathrm{softmax}(\tilde{A})V^{(l-1)}
\end{equation}
We introduce layer normalization (LN) before both the multi-head self-attention (MHA) mechanism and the feed-forward neural network (FFN). Moreover, for the feed-forward network, we unify the dimensions of the input, output, and the inner layer to the same value $d$. The MHA concatenates the representations from each attention head and maps them into a $d$-dimensional vector. Then we formally characterize the Transformer layer as below: 
\begin{equation}\footnotesize
\begin{split}
 h^{'(l)}=\text{MHA}(\mathrm{Attn}(h^{(l-1)}))+h^{(l-1)},\\
h^{(l)}=\mathrm{FFN}(\mathrm{LN}(h^{'(l)}))+h^{'(l)}
\end{split}
\end{equation}

The loss function follows SGCN \cite{derr2018signed}, which can be found in the Appendix. The training algorithm of encoding module is given in Algorithm \ref{alg:SE-SGformer}.

\begin{algorithm}\scriptsize
\caption{Training Algorithm of SE-SGformer}\label{alg:cap}
\begin{algorithmic}
\label{alg:SE-SGformer}
\STATE \textbf{Input:} A signed graph $\mathcal{G}=\left(\mathcal{V}, \mathcal{E}^{+}, \mathcal{E}^{-}\right)$; SE-SGformer model $f_{G}$; number of Transformer layers $L$; number of attention head $a$; max degree of positive or negative degrees $D$; number of random walks $r$; the length of random walk $l$; maximum path length $m$.
\STATE \textbf{Output:} node embedding $z_{i}$, $\forall v_{i}\in\mathcal{V}$.
\STATE Initialize the parameter of $f_{G}$.
\STATE Use spectral methods to generate initial node representation $x_i (\forall v_{i}\in\mathcal{V})$. 
\FOR {$v_i \in\mathcal{V}$}
    \STATE Obtain the $h_i^{(0)}$ by adding centrality encoding to node features in Equation \ref{eq:centrality}. 
\ENDFOR
\REPEAT
\FOR {$l$ = 1 to $L$}
    \STATE Obtain the adjacency matrix encoding by Equation \ref{eq:adjacency} and random walk encoding by Equation \ref{eq:random}.
    \STATE Obtain the attention score matrix $\tilde{A}$ by Equation \ref{eq:attention}.
    \STATE $h^{'(l)}=\text{MHA}(\mathrm{Attn}(h^{(l-1)})))+h^{(l-1)}$,
    \STATE $h^{(l)}=\mathrm{FFN}(\mathrm{LN}(h^{'(l)}))+h^{'(l)}$,
\ENDFOR
\STATE Update the parameters of $f_{G}$ based on loss (Equation \ref{eq:loss}).
\UNTIL{convergence}
\STATE $z_{i}\leftarrow\mathbf{h_{i}}^{(L)}, \forall v_{i}\in\mathcal{V}$
\end{algorithmic}
\end{algorithm}

\subsection{Explainable Prediction Module}
Intuitively, if the similarity between node $v_{i}$ and node $v_{j}$ is closer to the similarity between node $v_{j}$ and its most similar positive neighbors, there is likely a positive relationship between node $v_{i}$ and node $v_{j}$. Similarly, if the similarity between node $v_{i}$ and node $v_{j}$ is closer to the similarity between node $v_{j}$ and its least similar negative neighbors, then there is likely a negative relationship between node $v_{i}$ and node $v_{j}$. In this paper, we use Euclidean Distance to measure the similarity between nodes. The closer the distance between two nodes, the more similar they are. Let $z_i$ represent the embedding obtained after the Transformer Encoder encoding the node $v_{i}$. Then the distance of node $v_{i}$ and node $v_{j}$ can be calculated as follow:
\begin{equation}\footnotesize
    d_{ij}=\left\|z_i-z_j\right\|_2
    \label{eq:distance}
\end{equation}
We first generate the diffusion matrix $S$ for the entire graph through the SRWR algorithm \cite{jung2016personalized}. $S$ is based on the adjacency matrix and incorporates the potential positive and negative relationship information between nodes that the SRWR has uncovered. $S_{ij} = 1$ indicates a positive edge between node $v_i$ and node $v_j$, while $ S_{ij} = -1$ indicates a negative edge between them and $S_{ij}=0 $ signifies no relationship between the two nodes. The specific acquisition method of the diffusion matrix $S$ can be found in Appendix. For each node, we sample $n$ nodes from its positive neighbors using the adjacency matrix $A$, calculate the distance between the node and these neighbors, sort them in ascending order, and select $K$-nearest positive neighbors. Similarly, we also sample $n$ nodes from its negative neighbors, calculate the distance between the node and these neighbors, sort them in ascending order, and select $K$-farthest negative neighbors. If a node does not have enough $K$ negative neighbors, we sample its negative neighbors from the diffusion matrix $S$. 

Then we can identify the $K$-nearest (farthest) positive (negative) neighbor nodes of a given node. Let $\mathcal{K}_{ip}=\left\{v_{ip}^{1}, \ldots, v_{ip}^{k}\right\}$ be the set of $K$-nearest positive nodes of the node $v_{i}$ and $\mathcal{K}_{in}=\left\{v_{in}^{1}, \ldots, v_{in}^{k}\right\}$ be the set of $K$-farthest negative nodes of the node $v_{i}$. $d_{ip}$ denotes the median distance from $v_{i}$ to its $K$-nearest positive nodes:
\begin{equation}\footnotesize
    d_{ip} = \text{median} \left( \left\|z_{i}-z_j\right\|_{2}, v_{j} \in \mathcal{K}_{ip} \right)
\end{equation}
Similarly, $d_{in}$ denotes the median distance from $v_{i}$ to its $K$-farthest negative nodes:
\begin{equation}\footnotesize
    d_{in} = \text{median} \left( \left\|z_i-z_j\right\|_{2},  v_{j} \in \mathcal{K}_{in} \right)
\end{equation}
Then $d_{ij}$ denotes the distance from $v_{i}$ to $v_{j}$ which is calculated by Equation \ref{eq:distance}.
The result of link sign prediction can be obtained by comparing the distances between $d_{ij}$ and $d_{ip}$ or $d_{in}$. If $d_{ij}$ is closer to $d_{ip}$, the result of link sign prediction $\hat{y}_{ij} = 1$ or if $d_{ij}$ is closer to  $d_{in}$, then $\hat{y}_{ij} = -1$. 


\textbf{Time complexity analysis.}
The time complexity of Transformer primarily comes from the self-attention mechanism, which is $(O(|V|^2 \cdot d))$. The time complexity of the discriminate function can be analyzed by considering the key operations it performs. For each node, the function calculates the distances between the node's embedding and its neighbors' embeddings. For positive neighbors, it computes the distance for \( n \) sampled neighbors with a time complexity of \( O(n \cdot d) \). A similar calculation is performed for negative neighbors. Subsequently, the function selects the top \( K \) nearest neighbors from the \( n \) samples, which has a complexity of \( O(K \log n) \). These steps are repeated for each of the \( |V| \) nodes, leading to a total complexity of \( O(|V| \cdot (n \cdot d + K \log n)) \). 
The overall time complexity is \( O(|V| \cdot (n \cdot d + K \log n))+|V|^2 \cdot d)\).

\begin{table*}[htbp]\footnotesize
\centering
\begin{tabular}{lccccccc}
\toprule
\textbf{Dataset} & \textbf{GCN} & \textbf{GAT} & \textbf{SGCN} & \textbf{SNEA} & \textbf{SGCL} & \textbf{SIGformer} & \textbf{SE-SGformer} \\
\midrule
Amazon-music & 63.87 ± 3.57 & 65.39 ± 2.88 & 70.63 ± 0.69 & 70.48 ± 0.05 &\underline{78.26 ± 1.52}  & 58.64 ± 0.64 & \textbf{79.20 ± 0.23}† \\
Epinions & 71.07 ± 1.26 & 73.97 ± 1.64 &  \textbf{86.97 ± 1.53} & \underline{82.26 ± 0.57} & 70.83 ± 5.10 & 57.07 ± 0.38  & 72.84 ± 1.78 \\
KuaiRand & 44.35 ± 0.00 & 51.63 ± 2.84 & \textbf{62.85 ± 0.05} & \underline{61.95 ± 0.13} & 60.68 ± 0.99 & 61.40 ± 0.47 & 56.89 ± 0.12 \\
KuaiRec & 61.56 ± 0.42 & 65.73 ± 0.74 & \underline{85.11 ± 0.11} &  79.69 ± 0.01 & 79.84 ± 3.13 &  61.31 ± 1.37 & \textbf{85.60 ± 0.05}† \\
WikiRfa & 70.79 ± 6.02 & 71.31 ± 3.56 & \underline{78.69 ± 1.06} &75.20 ± 0.13 & 75.02 ± 4.33  & 65.60 ± 0.94 & \textbf{79.99 ± 0.08}† \\
WikiElec & 66.21 ± 1.50 & 66.50 ± 1.76 & 79.14 ± 0.48 & 77.10 ± 0.68 & \underline{79.63 ± 2.82}  & 65.74 ± 2.72 & \textbf{80.63 ± 0.08}† \\
Bitcoin-OTC & 83.77 ± 0.60 & 86.37 ± 1.24 & \underline{88.22 ± 0.69} & 86.05 ± 0.46 & 87.65 ± 1.74  & 80.30 ± 2.32 & \textbf{90.03 ± 0.35}‡ \\
Bitcoin-Alpha & 83.99 ± 1.61 & 86.25 ± 0.38 & \underline{87.96 ± 0.38} & 87.95 ± 0.15 & 83.01 ± 3.79 & 73.82 ± 3.55 & \textbf{89.88 ± 0.40}‡ \\
\bottomrule
\end{tabular}
\caption{Comparison of Accuracy(\%) across Different Models. The best scores are in bold, and the second-best ones are underlined. "†" and "‡" indicate the statistically significant improvements with $p < 0.05$ and $p < 0.01$ (one-sided paired t-test) over the best baseline, respectively.}
\label{tab:results}
\end{table*}

\section{Experiments}
In this section, we conduct experiments on real-world datasets to verify the effectiveness of SE-SGformer.
In particular, we aim to answer the following research questions: 
\begin{itemize}
    \item \textbf{RQ1}: Can SE-SGformer provide accurate predictions and explanatory information ?
    \item \textbf{RQ2}: How do the hyper-parameters affect the performance of SE-SGformer ?
    \item \textbf{RQ3}: How does each component of SE-SGformer contribute to the link sign prediction performance?
\end{itemize}

Our experimental datasets include Bitcoin-OTC, Bitcoin-Alpha, WikiElec, WikiRfa, Epinions, KuaiRand, KuaiRec and Amazon-music, with baseline methods being GCN \cite{kipf2016semi}, GAT \cite{velivckovic2017graph}, 
SGCN \cite{derr2018signed}, SNEA \cite{li2020learning}, SGCL \cite{shu2021sgcl}, and SIGFormer \cite{chen2024sigformer}. For detailed information on the datasets and baselines, please refer to the Appendix.

\textbf{Metrics.} Prediction accuracy and Explanation accuracy (Precision@$K$) are two metrics for evaluating the performance of link sign prediction and explanation performance. Prediction accuracy measures the overall correctness of the model in predicting whether an edge is positive or negative by calculating the proportion of correctly predicted edges. Also, we generate corresponding explanatory information for real-world datasets, which includes the $K$-nearest positive neighbors and $K$-farthest negative neighbors for each node as the ground truth for explanation. The specific process of generation is detailed in Appendix. Then precision@$K$ is the proportion of the nodes identified after sorting the neighbors of the nodes during the decision-making process, which constitutes the explanatory truth. 

\textbf{Configurations.}
All experiments were conducted on a 64-bit machine equipped with two NVIDIA GPUs (NVIDIA L20, 1440 MHz, 48 GB memory). For our SE-SGformer model, we used the Adam optimizer and performed a grid search to determine the hyperparameters. Specifically, we set the hidden embedding dimension \(d\) to 128, the learning rate to \(1 \times 10^{-3}\), the weight decay to \(5 \times 10^{-4}\), and the number of Transformer layers to $L = 1$. For the discriminator, we choose $K = 40$ and the number of randomly sampled neighbors $m = 200$. We searched for the optimal $L$ in the range [1, 4] with a step size of 1, $d$ in the range [16, 32, 64, 128], and max degree in the range [6, 8, 10, 12, 14].

\subsection{Performance and Explanation Quality (RQ1)}
To answer RQ1, we compare the performance of SE-SGformer with baselines on real-world datasets in terms of link sign prediction. All datasets were experimented with five times, and the link sign prediction accuracy and standard deviation are shown in Table \ref{tab:results}. From the table, we can observe the following results:
\begin{itemize}
    \item Our method outperforms SGCN, SNEA, and SGCL on most datasets, indicating that our encoding approach produces suitable graph representations and, combined with our explainable decisions, achieves excellent results. As mentioned earlier, our method can alleviate issues such as the scarcity of negative edges in signed graphs and the inability to learn appropriate representations for unbalanced cycles. Therefore, our method can achieve better performance.
    \item The performance of SE-SGformer far surpasses that of GCN, GAT, that is because our method makes better use of the information from negative edges
    compared to unsigned GNNs. 
    \item SIGformer is a relatively new signed graph Transformer model for recommendation systems, and our model outperforms SIGformer on most datasets. This also demonstrates the effectiveness of our unique encoding design.
\end{itemize}
Then, we evaluate the quality of the explanatory information of the $k$- nearest positive (farthest negative) neighbors identified in the explainable decision-making process. We set $K$ = 40 and compared our method with the baseline on three datasets. The precise@40 and standard deviation are shown in Table \ref{tab:exp_precision}. We can observe that our model achieves better explanatory accuracy compared to other methods, indicating that it performs well in identifying $K$-nearest positive ($K$-farthest negative) neighbors. This also reflects that our encoder is capable of learning suitable graph representations. The reason we set $K=40$ is that a node typically has many positive neighbors, hence a small value of $K$ could lead to the selection of nodes that are not representative, resulting in more errors. Also, as mentioned earlier, there are not many negative neighbors in the actual dataset, and our model can still achieve good explanatory accuracy with $K=40$. This indicates that our diffusion matrix effectively mitigates the issue of fewer negative edges in the datasets.


\begin{table}[htbp]\footnotesize
\centering
\resizebox{0.5\textwidth}{!}{
\begin{tabular}{lcccc}
\toprule
\textbf{Model} & \textbf{Bitcoin-OTC} & \textbf{Bitcoin-Alpha} & \textbf{Amazon-music} \\
\midrule
\textbf{SGCN}        & 57.25 ± 0.15 & 54.88 ± 0.14 & 60.78 ± 0.07 \\
\textbf{SNEA}        & 55.09 ± 0.16 & 55.43 ± 0.15 & 60.29 ± 0.07 \\
\textbf{SGCL}        & 54.10 ± 0.16 & 54.59 ± 0.15 & 61.70 ± 0.08 \\
\textbf{SIGformer}   & 53.33 ± 0.14 & 51.27 ± 0.09 & 58.76 ± 0.05 \\
\textbf{SE-SGformer} & \textbf{75.19 ± 0.13} & \textbf{94.47 ± 0.58} & \textbf{76.07 ± 0.20} \\
\bottomrule
\end{tabular}
}
\caption{The metric precision@40 (\%) of baselines on different datasets}
\label{tab:exp_precision}
\end{table}


\begin{figure}[htbp]
    \centering
    \begin{minipage}{0.49\linewidth}
        \centering
        \includegraphics[width=\linewidth]{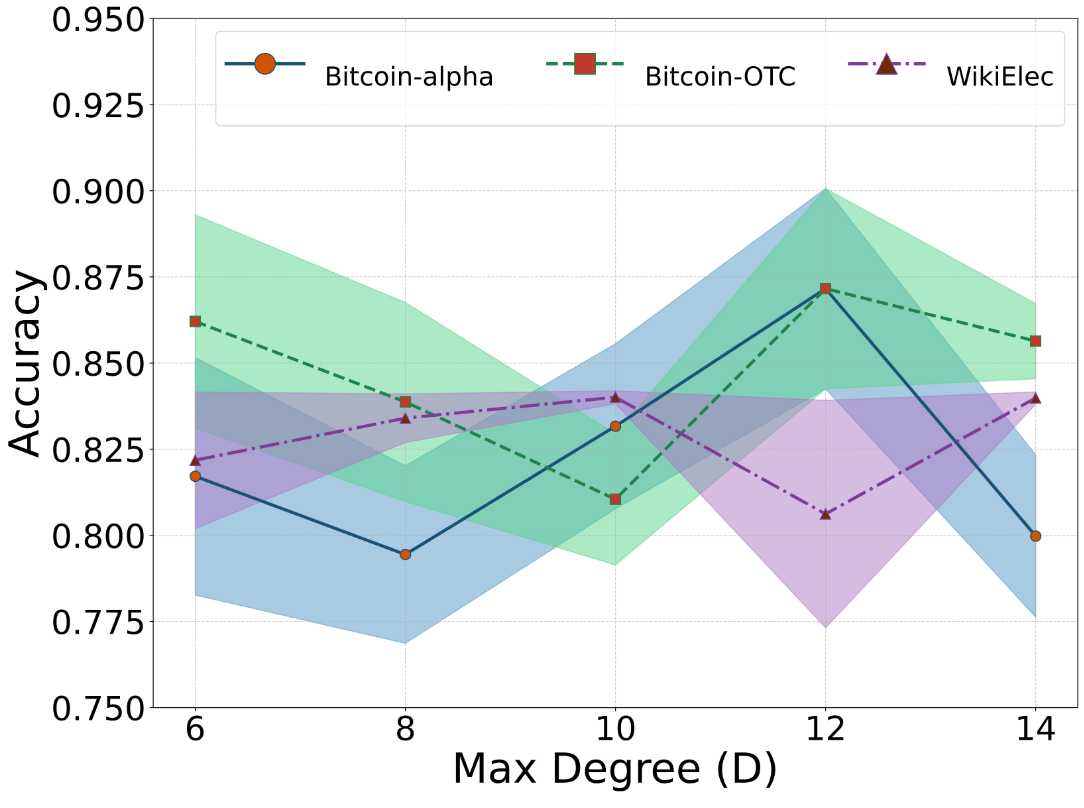}
        \label{fig:degree}
    \end{minipage}%
    \hfill
    \begin{minipage}{0.49\linewidth}
        \centering
        \includegraphics[width=\linewidth]{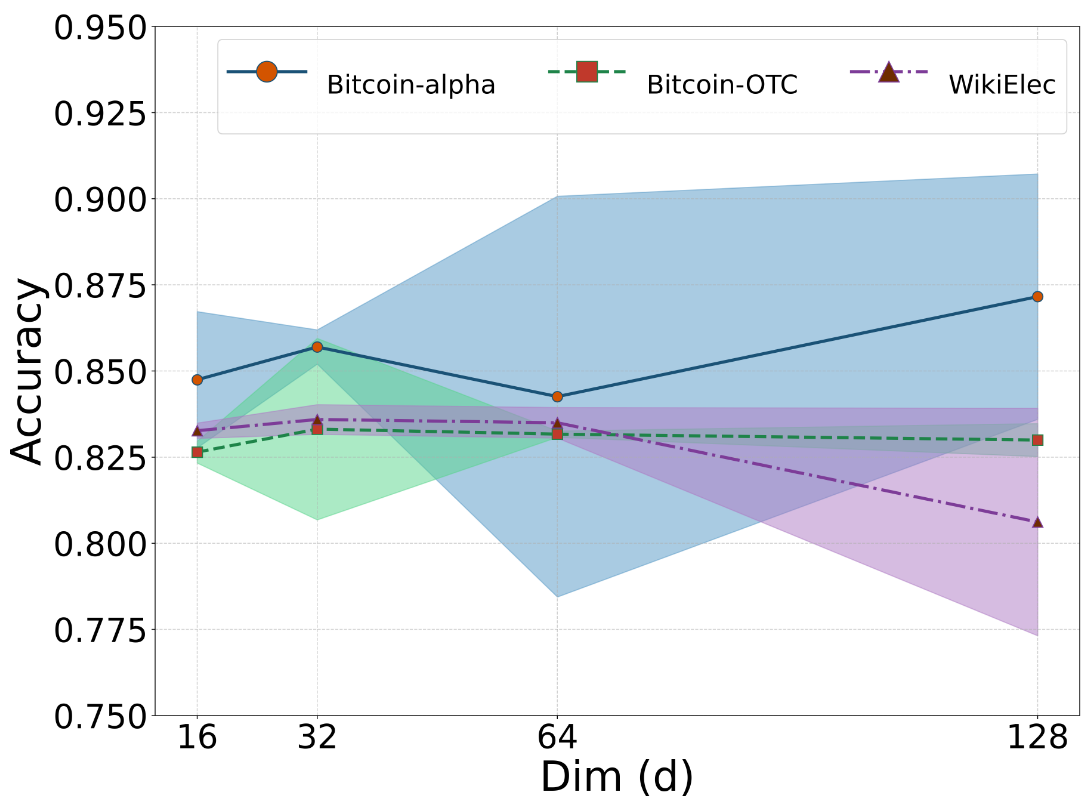}
        \label{fig:dim}
    \end{minipage}
    \caption{Parameter sensitivity analysis}
    \label{fig:sensitivity}
\end{figure}

\subsection{Hyperparameter Sensitivity Analysis (RQ2)}
In this subsection, we conduct a detailed sensitivity analysis of the key hyper-parameters max degree ($D$), dim ($d$), and layer ($L$). $D$ represents the maximum value of positive or negative degrees, $d$ refers to the dimension of the node embedding, and $L$ indicates the number of Transformer layers. We systematically vary these hyperparameters to assess their impact on the model's performance. In Figure \ref{fig:sensitivity}, $D$ is varied between 6 and 14, showing that while changes in $D$ slightly influence the model's performance, the effect is manifested as minor fluctuations across different datasets. $d$ is adjusted from 16 to 128 to explore how the dimension of the node embedding affects performance, with results indicating relatively stable performance and only slight deviations at larger dimensions. Lastly, $L$ is varied from 1 to 4, revealing that increasing the number of layers generally reduces the model's performance, detailed experimental results can be found in Appendix.

\begin{figure}[htbp]
    \centering
    \begin{minipage}{0.49\linewidth}
        \centering
        \includegraphics[width=\linewidth]{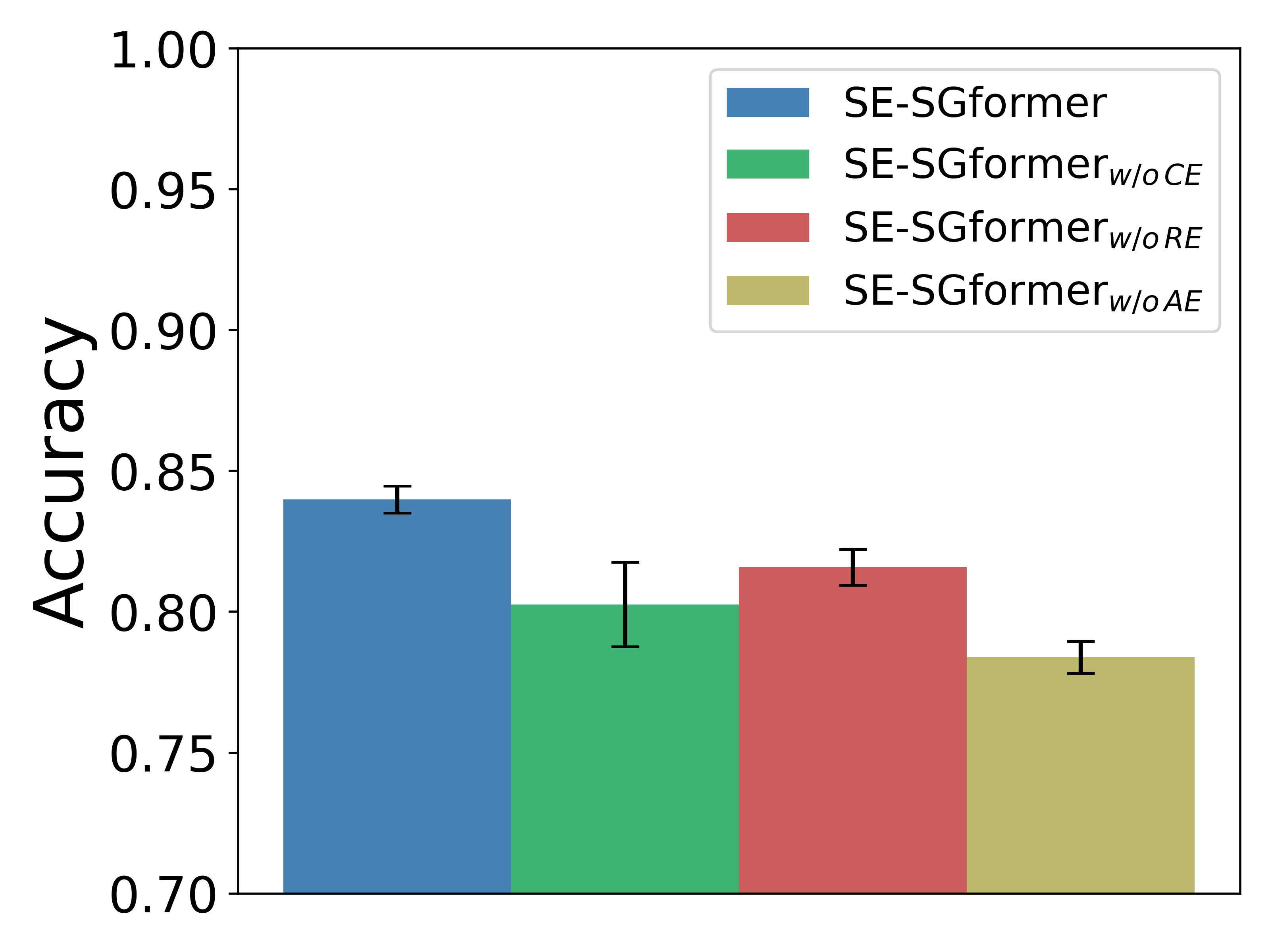}
        \label{fig:bitcoinOTC}
    \end{minipage}%
    \begin{minipage}{0.49\linewidth}
        \centering
        \includegraphics[width=\linewidth]{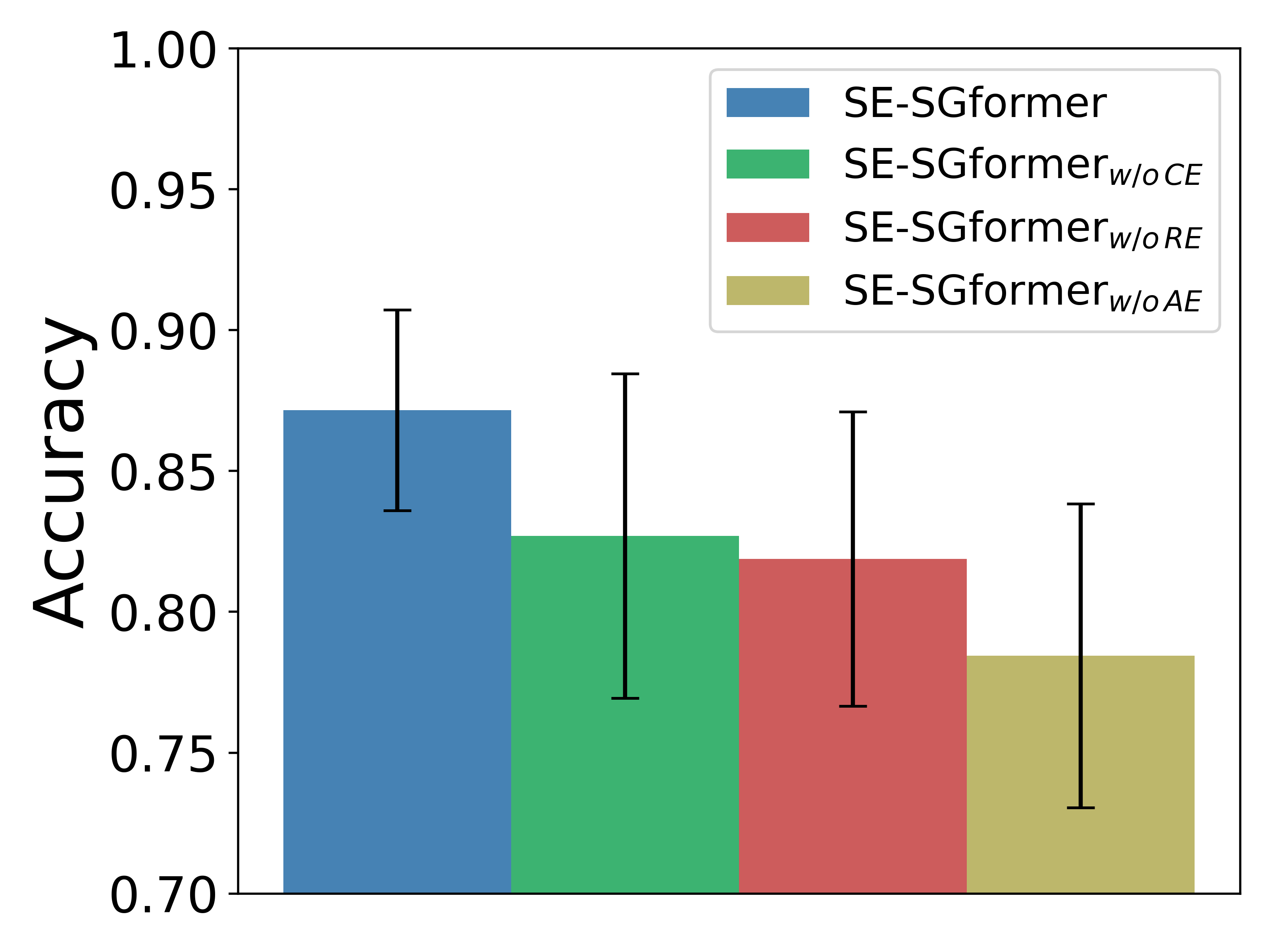}
        \label{fig:bitcoinAlpha}
    \end{minipage}
    \caption{Ablation study on Bitcoin-OTC (left) and Bitcoin-Alpha (right) datasets.}
    \label{fig:ablationStudy}
\end{figure}

\subsection{Ablation Study (RQ3)}
To answer RQ3, we conducted ablation experiments to explore the impact of the three encodings of Transformer on prediction accuracy. SE-SGformer\textsubscript{w/o-CE} denotes the variant without centrality encoding. SE-SGformer\textsubscript{w/o-RE} denotes the variant without signed random walk encoding. SE-SGformer\textsubscript{w/o-AE} denotes the variant without adjacency matrix encoding. The experiment results on Bitcoin-OTC and Bitcoin-Alpha datasets are reported in Figure \ref{fig:ablationStudy}. We can observe that the model's performance consistently decreases when each of the three types of encoding is removed, i.e., SE-SGformer\textsubscript{w/o-CE}, SE-SGformer\textsubscript{w/o-RE} and SE-SGformer\textsubscript{w/o-AE} exhibit significantly inferior performance than SE-SGformer on Bitcoin-OTC and Bitcoin-Alpha. This result clearly reflects the effectiveness of the three types of encoding we used. Among them, the model's performance drops the most when the adjacency matrix encoding is removed, indicating that the information about the direct relationships between nodes and their first-order neighbors is highly useful.

\begin{figure}[htbp]
    \centering
    \begin{minipage}{0.49\linewidth}
        \centering
    \includegraphics[width=1\linewidth]{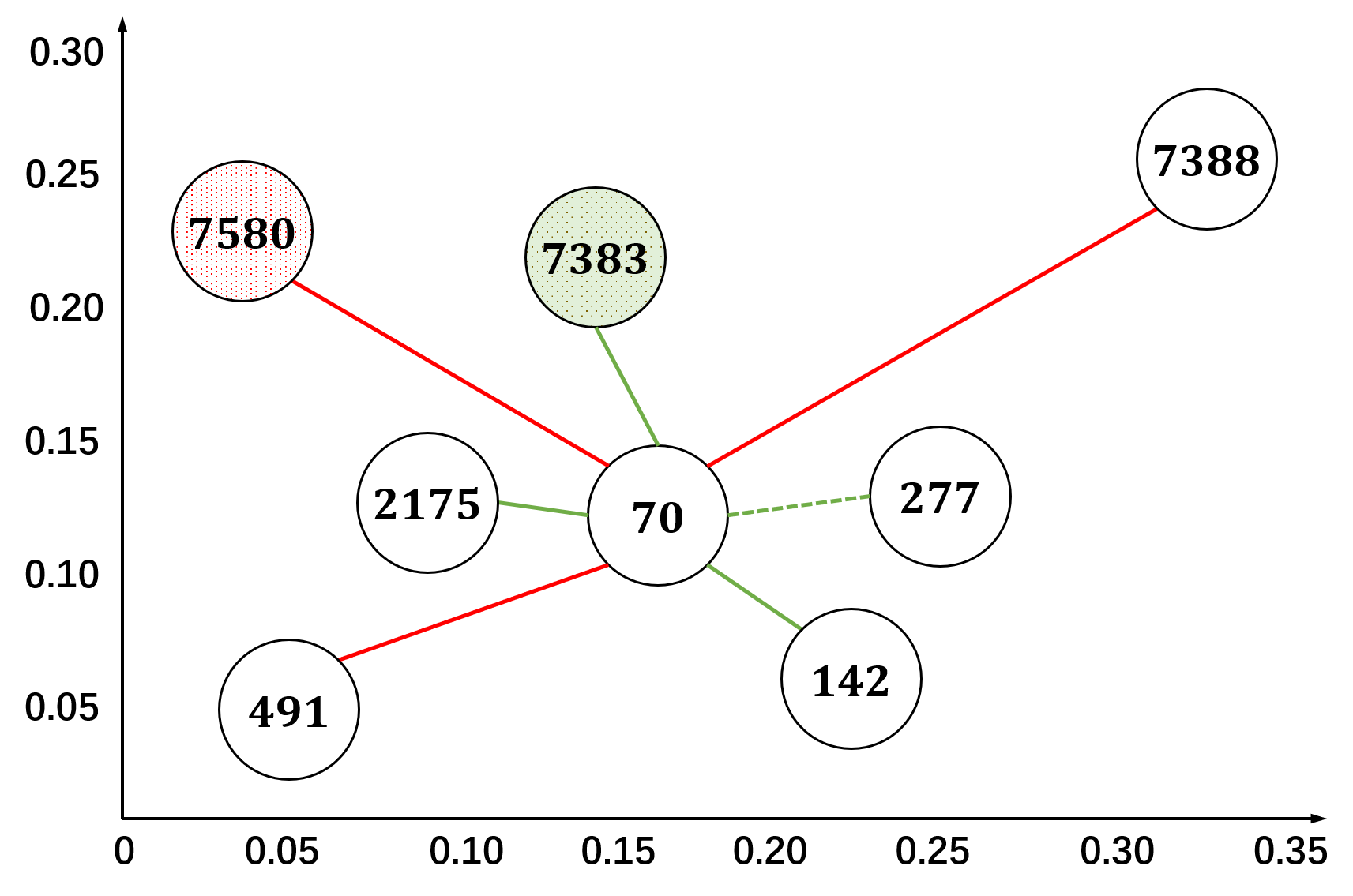}
    \label{fig:left}
    \end{minipage}%
    \begin{minipage}{0.49\linewidth}
        \centering
    \includegraphics[width=1\linewidth]{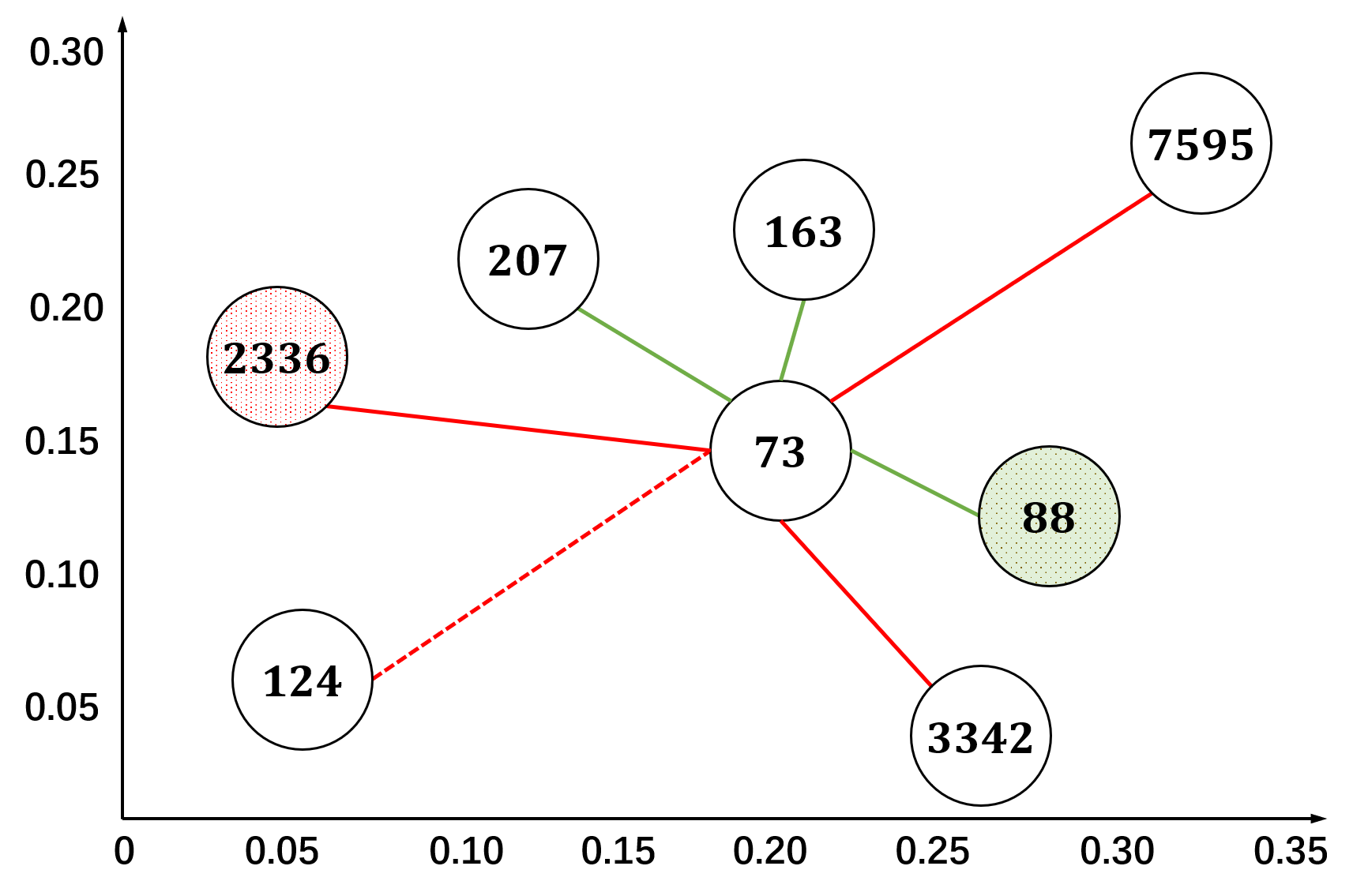}
    \label{fig:right}
    \end{minipage}
    \caption{Case study for node pairs with positive link (left) and negative link (right). Green lines represent positive edges, red lines represent negative edges, and dashed lines indicate the edges to be predicted.}
    \label{fig:CaseStudy}
\end{figure} 

\textbf{Case Study}. We conduct a case study using the Bitcoin-Alpha dataset to illustrate how our discriminator accurately determines the sign of an edge by identifying $k$ neighbors. Specifically, we analyze a pair of points where the ground truth of their edge is positive, showcasing this as a typical case for predicting a positive edge (Figure \ref{fig:CaseStudy} (left)). In the figure, we observe the median distance of three positive neighbors (e.g., point 7383) and the median distance of three negative neighbors (e.g., point 7580) for point 70. The distance between point 70 and point 277 is evidently closer to the median distance of the positive neighbors, leading to the prediction of this relationship as a positive edge. Similarly, we identified a typical case for predicting a negative edge (Figure \ref{fig:CaseStudy} (right)). The distance between point 73 and point 124  is closer to the median distance of the negative neighbors (e.g., point 2336), resulting in the prediction of a negative edge. 

\section{Conclusion}

In this paper, we address the challenge of self-explainable SGNNs by proposing SE-SGformer, a novel graph Transformer-based model for signed graphs. Our model predicts link signs by identifying $K$-nearest positive and $K$-farthest negative neighbors. We also theoretically prove that our signed random walk encoding is more powerful than the shortest path encoding and our Transformer architecture with signed random walk encoding is more powerful than SGCN.
Extensive experiments on real-world datasets validate our model's effectiveness. As the first exploration into the explainability of link sign prediction, we anticipate further research in this area.


\section*{Acknowledgment}
This study was supported by the National Natural Science Foundation of China (W2411020, 32170645). We thank the high-performance computing platform at the National Key Laboratory of Crop Genetic Improvement in Huazhong Agricultural University.

\bigskip

\bigskip
\bibliography{aaai25}

\begin{thebibliography}{40}
\providecommand{\natexlab}[1]{#1}

\bibitem[{Baldassarre and Azizpour(2019)}]{baldassarre2019explainability}
Baldassarre, F.; and Azizpour, H. 2019.
\newblock Explainability techniques for graph convolutional networks.
\newblock \emph{arXiv preprint arXiv:1905.13686}.

\bibitem[{Bonchi et~al.(2019)Bonchi, Galimberti, Gionis, Ordozgoiti, and Ruffo}]{bonchi2019discovering}
Bonchi, F.; Galimberti, E.; Gionis, A.; Ordozgoiti, B.; and Ruffo, G. 2019.
\newblock Discovering polarized communities in signed networks.
\newblock In \emph{Proceedings of the 28th acm international conference on information and knowledge management}, 961--970.

\bibitem[{Chen, O’Bray, and Borgwardt(2022)}]{chen2022structure}
Chen, D.; O’Bray, L.; and Borgwardt, K. 2022.
\newblock Structure-aware transformer for graph representation learning.
\newblock In \emph{International Conference on Machine Learning}, 3469--3489. PMLR.

\bibitem[{Chen et~al.(2024)Chen, Chen, Zhou, Wang, Han, Su, Yuan, and Wang}]{chen2024sigformer}
Chen, S.; Chen, J.; Zhou, S.; Wang, B.; Han, S.; Su, C.; Yuan, Y.; and Wang, C. 2024.
\newblock SIGformer: Sign-aware Graph Transformer for Recommendation.
\newblock \emph{arXiv preprint arXiv:2404.11982}.

\bibitem[{Chen et~al.(2018)Chen, Qian, Liu, and Sun}]{chen2018bridge}
Chen, Y.; Qian, T.; Liu, H.; and Sun, K. 2018.
\newblock " Bridge" Enhanced Signed Directed Network Embedding.
\newblock In \emph{Proceedings of the 27th acm international conference on information and knowledge management}, 773--782.

\bibitem[{Dai and Wang(2021)}]{dai2021towards}
Dai, E.; and Wang, S. 2021.
\newblock Towards self-explainable graph neural network.
\newblock In \emph{Proceedings of the 30th ACM International Conference on Information \& Knowledge Management}, 302--311.

\bibitem[{Deng and Shen(2024)}]{deng2024self}
Deng, J.; and Shen, Y. 2024.
\newblock Self-Interpretable Graph Learning with Sufficient and Necessary Explanations.
\newblock In \emph{Proceedings of the AAAI Conference on Artificial Intelligence}, volume~38, 11749--11756.

\bibitem[{Derr, Ma, and Tang(2018)}]{derr2018signed}
Derr, T.; Ma, Y.; and Tang, J. 2018.
\newblock Signed graph convolutional networks.
\newblock In \emph{2018 IEEE International Conference on Data Mining (ICDM)}, 929--934. IEEE.

\bibitem[{Huang et~al.(2019)Huang, Shen, Hou, and Cheng}]{huang2019signed}
Huang, J.; Shen, H.; Hou, L.; and Cheng, X. 2019.
\newblock Signed graph attention networks.
\newblock In \emph{International Conference on Artificial Neural Networks}, 566--577. Springer.

\bibitem[{Huang et~al.(2021)Huang, Shen, Hou, and Cheng}]{huang2021sdgnn}
Huang, J.; Shen, H.; Hou, L.; and Cheng, X. 2021.
\newblock SDGNN: Learning Node Representation for Signed Directed Networks.
\newblock \emph{arXiv preprint arXiv:2101.02390}.

\bibitem[{Javari et~al.(2020)Javari, Derr, Esmailian, Tang, and Chang}]{javari2020rose}
Javari, A.; Derr, T.; Esmailian, P.; Tang, J.; and Chang, K. C.-C. 2020.
\newblock Rose: Role-based signed network embedding.
\newblock In \emph{Proceedings of The Web Conference 2020}, 2782--2788.

\bibitem[{Jung et~al.(2016)Jung, Jin, Sael, and Kang}]{jung2016personalized}
Jung, J.; Jin, W.; Sael, L.; and Kang, U. 2016.
\newblock Personalized ranking in signed networks using signed random walk with restart.
\newblock In \emph{2016 IEEE 16th International Conference on Data Mining (ICDM)}, 973--978. IEEE.

\bibitem[{Jung, Yoo, and Kang(2020)}]{jung2020signed}
Jung, J.; Yoo, J.; and Kang, U. 2020.
\newblock Signed Graph Diffusion Network.
\newblock \emph{arXiv preprint arXiv:2012.14191}.

\bibitem[{Kim et~al.(2018)Kim, Park, Lee, and Kang}]{kim2018side}
Kim, J.; Park, H.; Lee, J.-E.; and Kang, U. 2018.
\newblock Side: representation learning in signed directed networks.
\newblock In \emph{Proceedings of the 2018 World Wide Web Conference}, 509--518.

\bibitem[{Kipf and Welling(2016)}]{kipf2016semi}
Kipf, T.~N.; and Welling, M. 2016.
\newblock Semi-supervised classification with graph convolutional networks.
\newblock \emph{arXiv preprint arXiv:1609.02907}.

\bibitem[{Li et~al.(2020)Li, Tian, Zhang, and Chang}]{li2020learning}
Li, Y.; Tian, Y.; Zhang, J.; and Chang, Y. 2020.
\newblock Learning signed network embedding via graph attention.
\newblock In \emph{Proceedings of the AAAI Conference on Artificial Intelligence}, 4772--4779.

\bibitem[{Luo et~al.(2020)Luo, Cheng, Xu, Yu, Zong, Chen, and Zhang}]{luo2020parameterized}
Luo, D.; Cheng, W.; Xu, D.; Yu, W.; Zong, B.; Chen, H.; and Zhang, X. 2020.
\newblock Parameterized explainer for graph neural network.
\newblock \emph{Advances in neural information processing systems}, 33: 19620--19631.

\bibitem[{Ni et~al.(2024)Ni, Wang, Zhang, Li, Zheng, Denny, and Liu}]{ni2024enhancing}
Ni, L.; Wang, S.; Zhang, Z.; Li, X.; Zheng, X.; Denny, P.; and Liu, J. 2024.
\newblock Enhancing student performance prediction on learnersourced questions with sgnn-llm synergy.
\newblock In \emph{Proceedings of the AAAI Conference on Artificial Intelligence}, volume~38, 23232--23240.

\bibitem[{Pope et~al.(2019)Pope, Kolouri, Rostami, Martin, and Hoffmann}]{pope2019explainability}
Pope, P.~E.; Kolouri, S.; Rostami, M.; Martin, C.~E.; and Hoffmann, H. 2019.
\newblock Explainability methods for graph convolutional neural networks.
\newblock In \emph{Proceedings of the IEEE/CVF conference on computer vision and pattern recognition}, 10772--10781.

\bibitem[{Rong et~al.(2020)Rong, Bian, Xu, Xie, Wei, Huang, and Huang}]{rong2020self}
Rong, Y.; Bian, Y.; Xu, T.; Xie, W.; Wei, Y.; Huang, W.; and Huang, J. 2020.
\newblock Self-supervised graph transformer on large-scale molecular data.
\newblock \emph{Advances in neural information processing systems}, 33: 12559--12571.

\bibitem[{Schnake et~al.(2021)Schnake, Eberle, Lederer, Nakajima, Sch{\"u}tt, M{\"u}ller, and Montavon}]{schnake2021higher}
Schnake, T.; Eberle, O.; Lederer, J.; Nakajima, S.; Sch{\"u}tt, K.~T.; M{\"u}ller, K.-R.; and Montavon, G. 2021.
\newblock Higher-order explanations of graph neural networks via relevant walks.
\newblock \emph{IEEE transactions on pattern analysis and machine intelligence}, 44(11): 7581--7596.

\bibitem[{Seo, Kim, and Park(2024)}]{seo2024interpretable}
Seo, S.; Kim, S.; and Park, C. 2024.
\newblock Interpretable prototype-based graph information bottleneck.
\newblock \emph{Advances in Neural Information Processing Systems}, 36.

\bibitem[{Shu et~al.(2021)Shu, Du, Chang, Chen, Zheng, Xing, and Shen}]{shu2021sgcl}
Shu, L.; Du, E.; Chang, Y.; Chen, C.; Zheng, Z.; Xing, X.; and Shen, S. 2021.
\newblock SGCL: Contrastive Representation Learning for Signed Graphs.
\newblock In \emph{Proceedings of the 30th ACM International Conference on Information \& Knowledge Management}, 1671--1680.

\bibitem[{Tang, Aggarwal, and Liu(2016)}]{tang2016node}
Tang, J.; Aggarwal, C.; and Liu, H. 2016.
\newblock Node classification in signed social networks.
\newblock In \emph{Proceedings of the 2016 SIAM international conference on data mining}, 54--62. SIAM.

\bibitem[{Vaswani et~al.(2017)Vaswani, Shazeer, Parmar, Uszkoreit, Jones, Gomez, Kaiser, and Polosukhin}]{vaswani2017attention}
Vaswani, A.; Shazeer, N.; Parmar, N.; Uszkoreit, J.; Jones, L.; Gomez, A.~N.; Kaiser, {\L}.; and Polosukhin, I. 2017.
\newblock Attention is all you need.
\newblock \emph{Advances in neural information processing systems}, 30.

\bibitem[{Veli{\v{c}}kovi{\'c} et~al.(2017)Veli{\v{c}}kovi{\'c}, Cucurull, Casanova, Romero, Lio, and Bengio}]{velivckovic2017graph}
Veli{\v{c}}kovi{\'c}, P.; Cucurull, G.; Casanova, A.; Romero, A.; Lio, P.; and Bengio, Y. 2017.
\newblock Graph attention networks.
\newblock \emph{arXiv preprint arXiv:1710.10903}.

\bibitem[{Vu and Thai(2020)}]{vu2020pgm}
Vu, M.; and Thai, M.~T. 2020.
\newblock Pgm-explainer: Probabilistic graphical model explanations for graph neural networks.
\newblock \emph{Advances in neural information processing systems}, 33: 12225--12235.

\bibitem[{Wang et~al.(2018)Wang, Zhang, Hou, Xie, Guo, and Liu}]{wang2018shine}
Wang, H.; Zhang, F.; Hou, M.; Xie, X.; Guo, M.; and Liu, Q. 2018.
\newblock Shine: Signed heterogeneous information network embedding for sentiment link prediction.
\newblock In \emph{Proceedings of the eleventh ACM international conference on web search and data mining}, 592--600.

\bibitem[{Yeh, Chen, and Chen(2023)}]{yeh2023random}
Yeh, P.-K.; Chen, H.-W.; and Chen, M.-S. 2023.
\newblock Random walk conformer: learning graph representation from long and short range.
\newblock In \emph{Proceedings of the AAAI Conference on Artificial Intelligence}, volume~37, 10936--10944.

\bibitem[{Ying et~al.(2021)Ying, Cai, Luo, Zheng, Ke, He, Shen, and Liu}]{ying2021Transformers}
Ying, C.; Cai, T.; Luo, S.; Zheng, S.; Ke, G.; He, D.; Shen, Y.; and Liu, T.-Y. 2021.
\newblock Do transformers really perform badly for graph representation?
\newblock \emph{Advances in neural information processing systems}, 34: 28877--28888.

\bibitem[{Ying et~al.(2019)Ying, Bourgeois, You, Zitnik, and Leskovec}]{ying2019gnnexplainer}
Ying, Z.; Bourgeois, D.; You, J.; Zitnik, M.; and Leskovec, J. 2019.
\newblock Gnnexplainer: Generating explanations for graph neural networks.
\newblock \emph{Advances in neural information processing systems}, 32.

\bibitem[{Yu et~al.(2020)Yu, Xu, Rong, Bian, Huang, and He}]{yu2020graph}
Yu, J.; Xu, T.; Rong, Y.; Bian, Y.; Huang, J.; and He, R. 2020.
\newblock Graph information bottleneck for subgraph recognition.
\newblock \emph{arXiv preprint arXiv:2010.05563}.

\bibitem[{Yuan et~al.(2020)Yuan, Tang, Hu, and Ji}]{yuan2020xgnn}
Yuan, H.; Tang, J.; Hu, X.; and Ji, S. 2020.
\newblock Xgnn: Towards model-level explanations of graph neural networks.
\newblock In \emph{Proceedings of the 26th ACM SIGKDD international conference on knowledge discovery \& data mining}, 430--438.

\bibitem[{Yuan et~al.(2022)Yuan, Yu, Gui, and Ji}]{yuan2022explainability}
Yuan, H.; Yu, H.; Gui, S.; and Ji, S. 2022.
\newblock Explainability in graph neural networks: A taxonomic survey.
\newblock \emph{IEEE transactions on pattern analysis and machine intelligence}, 45(5): 5782--5799.

\bibitem[{Yuan, Wu, and Xiang(2017)}]{yuan2017sne}
Yuan, S.; Wu, X.; and Xiang, Y. 2017.
\newblock SNE: signed network embedding.
\newblock In \emph{Pacific-Asia conference on knowledge discovery and data mining}, 183--195. Springer.

\bibitem[{Zhang et~al.(2024)Zhang, Wu, Yuan, Pan, Tong, and Pei}]{zhang2024trustworthy}
Zhang, H.; Wu, B.; Yuan, X.; Pan, S.; Tong, H.; and Pei, J. 2024.
\newblock Trustworthy graph neural networks: aspects, methods, and trends.
\newblock \emph{Proceedings of the IEEE}.

\bibitem[{Zhang et~al.(2023{\natexlab{a}})Zhang, Liu, Zhao, Yang, Zheng, and Wang}]{zhang2023contrastive}
Zhang, Z.; Liu, J.; Zhao, K.; Yang, S.; Zheng, X.; and Wang, Y. 2023{\natexlab{a}}.
\newblock Contrastive Learning for Signed Bipartite Graphs.
\newblock In \emph{Proceedings of the 46th International ACM SIGIR Conference on Research and Development in Information Retrieval}, 1629--1638.

\bibitem[{Zhang et~al.(2023{\natexlab{b}})Zhang, Liu, Zheng, Wang, Han, Wang, Zhao, and Zhang}]{zhang2023rsgnn}
Zhang, Z.; Liu, J.; Zheng, X.; Wang, Y.; Han, P.; Wang, Y.; Zhao, K.; and Zhang, Z. 2023{\natexlab{b}}.
\newblock RSGNN: A Model-agnostic Approach for Enhancing the Robustness of Signed Graph Neural Networks.
\newblock In \emph{Proceedings of the ACM Web Conference 2023}, 60--70.

\bibitem[{Zhang et~al.(2022)Zhang, Liu, Wang, Lu, and Lee}]{zhang2022protgnn}
Zhang, Z.; Liu, Q.; Wang, H.; Lu, C.; and Lee, C. 2022.
\newblock Protgnn: Towards self-explaining graph neural networks.
\newblock In \emph{Proceedings of the AAAI Conference on Artificial Intelligence}, volume~36, 9127--9135.

\bibitem[{Zhu et~al.(2023)Zhu, Wen, Song, Wang, and Zheng}]{zhu2023structural}
Zhu, W.; Wen, T.; Song, G.; Wang, L.; and Zheng, B. 2023.
\newblock On structural expressive power of graph transformers.
\newblock In \emph{Proceedings of the 29th ACM SIGKDD Conference on Knowledge Discovery and Data Mining}, 3628--3637.

\end{thebibliography}

\section*{Reproducibility Checklist}
This paper:
\begin{itemize}
    \item Includes a conceptual outline and/or pseudocode description of AI methods introduced (\textbf{yes})
    \item Clearly delineates statements that are opinions, hypothesis, and speculation from objective facts and results (\textbf{yes})
    \item Provides well marked pedagogical references for less-familiare readers to gain background necessary to replicate the paper (\textbf{yes})
\end{itemize}

\noindent Does this paper make theoretical contributions? (\textbf{yes})

\noindent If yes, please complete the list below.
\begin{itemize}
    \item All assumptions and restrictions are stated clearly and formally. (\textbf{yes})
    \item All novel claims are stated formally (e.g., in theorem statements). (\textbf{yes})
    \item Proofs of all novel claims are included. (\textbf{yes})
    \item Proof sketches or intuitions are given for complex and/or novel results. (\textbf{yes})
    \item Appropriate citations to theoretical tools used are given. (\textbf{yes})
    \item All theoretical claims are demonstrated empirically to hold. (\textbf{yes})
    \item All experimental code used to eliminate or disprove claims is included. (\textbf{yes})
\end{itemize}

\noindent Does this paper rely on one or more datasets? (\textbf{yes})

\noindent If yes, please complete the list below
\begin{itemize}
    \item A motivation is given for why the experiments are conducted on the selected datasets (\textbf{yes})
    \item All novel datasets introduced in this paper are included in a data appendix. (\textbf{yes})
    \item All novel datasets introduced in this paper will be made publicly available upon publication of the paper with a license that allows free usage for research purposes. (\textbf{NA})
    \item All datasets drawn from the existing literature (potentially including authors’ own previously published work) are accompanied by appropriate citations. (\textbf{yes})
    \item All datasets drawn from the existing literature (potentially including authors’ own previously published work) are publicly available. (\textbf{yes})
    \item All datasets that are not publicly available are described in detail, with explanation why publicly available alternatives are not scientifically satisficing. (\textbf{NA})
\end{itemize}

\noindent Does this paper include computational experiments? (\textbf{yes})

\noindent If yes, please complete the list below.
\begin{itemize}
    \item Any code required for pre-processing data is included in the appendix. (\textbf{yes}).
    \item All source code required for conducting and analyzing the experiments is included in a code appendix. (\textbf{yes})
    \item All source code required for conducting and analyzing the experiments will be made publicly available upon publication of the paper with a license that allows free usage for research purposes. (\textbf{yes})
    \item All source code implementing new methods have comments detailing the implementation, with references to the paper where each step comes from (\textbf{partial})
    \item If an algorithm depends on randomness, then the method used for setting seeds is described in a way sufficient to allow replication of results. (\textbf{yes})
    \item This paper specifies the computing infrastructure used for running experiments (hardware and software), including GPU/CPU models; amount of memory; operating system; names and versions of relevant software libraries and frameworks. (\textbf{yes})
    \item This paper formally describes evaluation metrics used and explains the motivation for choosing these metrics. (\textbf{yes})
    \item This paper states the number of algorithm runs used to compute each reported result. (\textbf{yes})
    \item Analysis of experiments goes beyond single-dimensional summaries of performance (e.g., average; median) to include measures of variation, confidence, or other distributional information. (\textbf{yes})
    \item The significance of any improvement or decrease in performance is judged using appropriate statistical tests (e.g., Wilcoxon signed-rank). (yes/partial/no)
    \item This paper lists all final (hyper-)parameters used for each model/algorithm in the paper’s experiments. (\textbf{yes})
    \item This paper states the number and range of values tried2 per (hyper-) parameter during development of the paper, along with the criterion used for selecting the final parameter setting. (\textbf{yes})
\end{itemize}

\clearpage

\section*{Appendix}

\setcounter{theorem}{0}

\subsection{Related Work}
In this section, we review related works. First, we review signed graph representation learning methods. Then, we review
Graph Transformer methods. Finally, we review explainable graph neural networks.

\textbf{Signed Graph Representation Learning}. The widespread use of social media has made signed networks ubiquitous, sparking interest in their network representations \cite{chen2018bridge, wang2018shine, zhang2023contrastive, zhang2023rsgnn}. While link sign prediction dominates current research, tasks such as node classification \cite{tang2016node}, ranking \cite{jung2016personalized}, and community detection \cite{bonchi2019discovering} remain underexplored. Traditional embedding methods like SNE \cite{yuan2017sne}, SIDE \cite{kim2018side}, SGDN \cite{jung2020signed}, and ROSE \cite{javari2020rose} rely on random walks and linear probabilities but may overlook deeper relationships. Neural networks, particularly GCN-based SGNNs such as SGCN \cite{derr2018signed} and GS-GNN, along with GAT-based models like SiGAT \cite{huang2019signed}, SNEA \cite{li2020learning}, SDGNN \cite{huang2021sdgnn}, and SGCL \cite{shu2021sgcl}, are now being applied to capture these complexities, advancing signed graph representation learning.

\textbf{Graph Transformers}. Current GNN methods mainly based on the message passing paradigm suffer from oversmoothing and long-range modeling issues. An important attempt to alleviate these problems is to employ Transformers for graph representation learning \cite{rong2020self, ying2021Transformers, chen2022structure,chen2024sigformer}. The success of Transformer in this domain usually relies on positional encodings that integrate graph structural information into the Transformer framework. For example, Graphormer \cite{ying2021Transformers}, built on the standard Transformer architecture, effectively encodes the structural information of graphs into the model by introducing centrality encoding, spatial encoding, and edge encoding, and has achieved excellent results on a variety of graph representation learning tasks. Graphormer's several encodings can capture graph structure information but are not suitable for signed graphs. SIGformer \cite{chen2024sigformer} utilizes the Transformer architecture and introduces sign-aware spectral encoding and sign-aware path encoding, which respectively capture the spectral properties and path patterns of the signed graph for recommendation in signed graphs. However, its sign-aware path encoding faces the issue of dimensionality explosion when the paths become too long. Therefore, using a Transformer as a signed graph encoder to learn appropriate graph representations still poses significant challenges.

\textbf{Explainability in Graph Neural Networks}. Graph Neural Networks (GNNs) are powerful tools for graph-based machine learning tasks, but their predictions are often notexplainable for humans, which limits their applicability in critical or sensitive domains. The current approaches aimed at improving the interpretability of GNNs can be categorized into post-hoc explanation methods and self-interpretable models. Post-hoc explanations usually treat GNNs as black boxes and then learn an explainer to explain the outputs of a trained GNN, which include gradient/feature-based methods \cite{baldassarre2019explainability,pope2019explainability}, perturbation-based methods \cite{ying2019gnnexplainer}, surrogate methods\cite{vu2020pgm}, decomposition methods \cite{schnake2021higher}, and generation methods\cite{yuan2020xgnn}. For example, GNNExplainer \cite{ying2019gnnexplainer} takes a trained GNN and its prediction, and it selects a small subgraph of the input graph together with a small subset of node features that are most influential for the prediction as an explanation. PGM-Explainer \cite{vu2020pgm} perturbs node features to obtain similar samples, then employs anexplainable Bayesian network to provide explanations. However, post-hoc explanations are not directly obtained from the GNNs, which may lead to inaccuracies or incompleteness when revealing the actual reasoning process of the original model. While self-interpretable models generate explanations by the model themselves, including contribution estimation, the introduction ofexplainable modules \cite{dai2021towards}, embedding prototype learning\cite{zhang2022protgnn}, and rationale generation \cite{yu2020graph}. For instance, SE-GNN \cite{dai2021towards} can find K-nearest labeled nodes based on node similarity and local structural similarity for each unlabeled node to give explainable node classification. ProtGNN \cite{zhang2022protgnn} combines prototype learning with GNNs, whose interpretability on a specific input graph comes from displaying the learned prototypes and similar subgraphs in the input graph. GIB \cite{yu2020graph} can generate subgraphs that share the most similar property to the input graphs based on the information bottleneck theory. 
Although these self-interpretable models reduce the risk of misinterpretation of the true decision-making rationale compared to post-hoc methods that require additional explainers, they are primarily tailored for unsigned graph neural networks. Furthermore, most of them focus onexplainable solutions for graph-level tasks and node classification tasks. Our work differs from the aforementioned self-explanatory models. We aim to develop a self-interpretable SGNN( Signed Graph Neural Networks) that can perform link sign prediction tasks and simultaneously provide explanations for its predictions.

\begin{figure}
    \centering
    \includegraphics[width=1\linewidth]{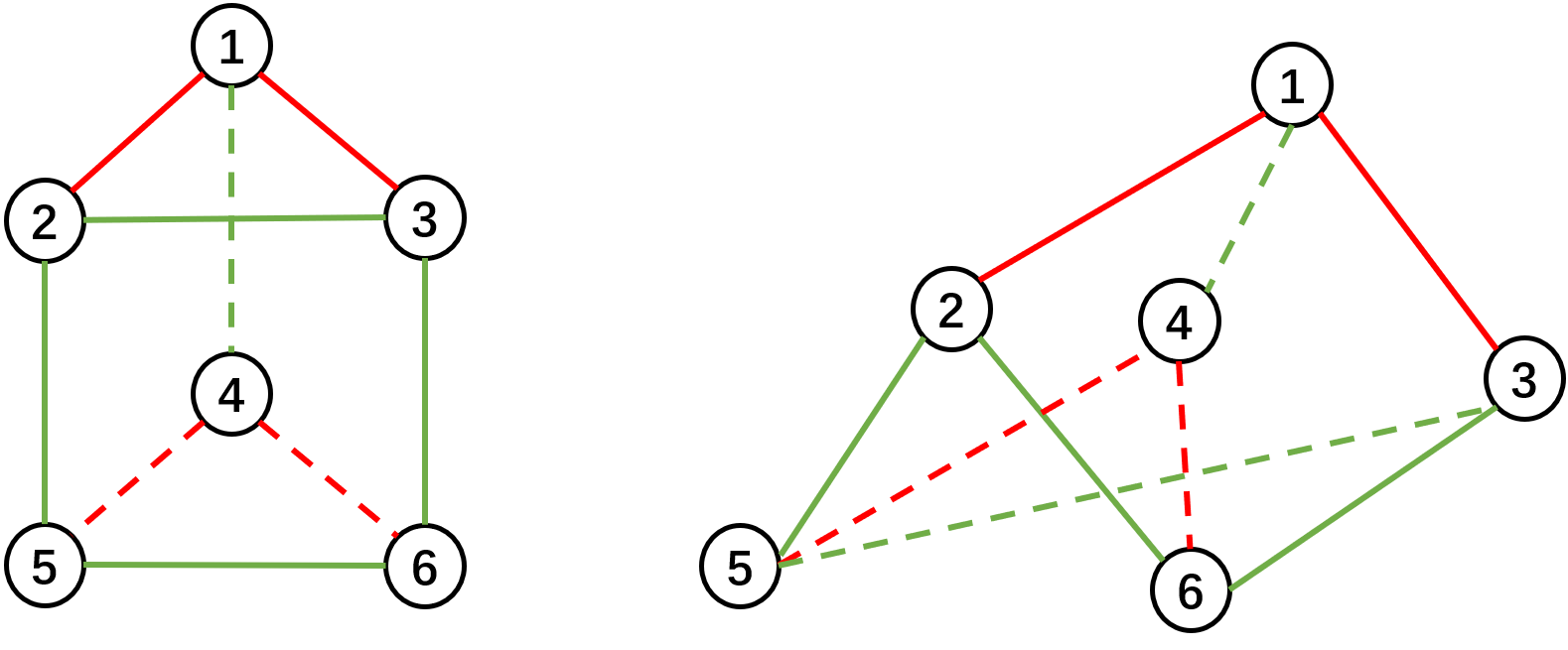}
    \caption{Signed Random walk encoding can decide the two graphs
as non-isomorphic while the Extended WL-test identify them as isomorphic.}
    \label{fig:enter-label}
\end{figure}

\subsection{Theorem 1 Experimental Verification}
We perform link sign prediction using the shortest path encoding(SPE)-based graph transformer and SE-SGformer on WikiElec, Bitcoin-OTC and Bitcoin-Alpha. The results are shown in Table \ref{tab:performance_comparison} below. We can see that SE-SGformer outperforms the SPE-based graph transformer in link sign prediction accuracy (Acc.) across all three datasets, which experimentally validates the conclusion in Theorem 1.
\begin{table}[h]
\centering
\resizebox{\linewidth}{!}{
\begin{tabular}{lcc}
\toprule
\textbf{Dataset} & \textbf{SE-SGformer} & \textbf{SPE-based graph transformer} \\
\midrule
WikiElec & \textbf{80.63} & 76.99 \\
Bitcoin-OTC & \textbf{90.03} & 85.02 \\
Bitcoin-Alpha & \textbf{89.88} & 86.18 \\
\bottomrule
\end{tabular}}
\caption{Performance comparison between SE-SGformer and SPE-based graph transformer on different datasets.}
\label{tab:performance_comparison}
\end{table}

\subsection{Proof of Theorem 2}
To analyze the expressive capabilities of the graph transformer architecture based on signed random walk encoding and SGCN, we introduce the definition of Extended WL-test For Signed Graph \cite{zhang2023rsgnn}:
\begin{definition}[Extended WL-test For Signed Graph] 
For a node $v_i$, its $l$-balanced (unbalanced) reach set $\mathcal{B}(l)(\mathcal{U}(l)))$ is defined as a
set of nodes that have even (odd) negative edges along a path. Based on the message-passing mechanism of SGNNs,  the process of extended WL-test for signed graph can be defined as below. For the first-iteration update, i.e. $l$= 1, the WL node label of a node $v_{i}$ is $(X_{i}^{(1)}(\mathcal{B}), X_{i}^{(1)}(\mathcal{U}))$ where:
\begin{equation}
\tiny
    \begin{gathered}
X_{i}^{(1)}\left(\mathcal{B}\right) =\varphi\left(\left\{\left(X_{i}^{(0)},\{X_{j}^{(0)}:v_{j}\in\mathcal{N}_{i}^{+}\}\right)\right\}\right) \\
X_{i}^{(1)}(\mathcal{U}) =\varphi\left(\left\{\left(X_{i}^{(0)} , \{X_{j}^{(0)} :v_{j}\in\mathcal{N}_{i}^{-}\}\right)\right\}\right) 
\end{gathered}
\end{equation}
For $l > 1$, the WL node label of a node $v_{i}$ is $(X_{i}^{(1)}(\mathcal{B}), X_{i}^{(1)}(\mathcal{U}))$ where:
\begin{equation}
\tiny
\begin{aligned}
X_{i}^{(\ell)}\left(\mathcal{B}\right) & = \varphi\Bigg(\big\{ \big(X_{i}^{(\ell-1)}\left(\mathcal{B}\right), \{X_{j}^{(\ell-1)}\left(\mathcal{B}\right) : v_{j} \in \mathcal{N}_{i}^{+}\}, \\
&\quad \{X_{j}^{(\ell-1)}\left({\mathcal U}\right) : v_{j} \in \mathcal{N}_{i}^{-}\}\big\}\Bigg), \\
X_{i}^{(\ell)}(\mathcal{U}) & = \varphi\Bigg(\big\{ \big(X_{i}^{(\ell-1)}\left(\mathcal{U}\right), \{X_{j}^{(\ell-1)}\left(\mathcal{U}\right) : v_{j} \in \mathcal{N}_{i}^{+}\}, \\
&\quad \{X_{j}^{(\ell-1)}\left(\mathcal{B}\right) : v_{j} \in \mathcal{N}_{i}^{-}\}\big\}\Bigg).
\end{aligned}
\end{equation}
where $\varphi$ is an injective function.
\end{definition}
The extended WL-test is defined with a similar aggregation and update process as a SGCN, and thus can be used to capture the expressibility of the
SGCN. Then based on conclusions proposed by \cite{zhu2023structural}, we can redefine the fundamental challenge of determining the expressive power of graph Transformers as understanding the expressive power of the SEG-WL test, which is inherently shaped by the design of these structural encodings. The Extended WL-test can be regard as a special SEG-WL test. 
Therefore, we can transform the problem of the expressive power of the our Transformer with random walk encoding and SGCN into a comparison of the expressive power between our random walk encoding and the extended WL-test.
So we first prove the following theory:
\renewcommand{\thetheorem}{3}
\begin{theorem}
\label{the:extended}
With a sufficient number of random walk length, signed random walk encoding is more expressive than the Extended WL-test For Signed Graph.
\end{theorem}
\begin{proof}
We illustrate our proof by considering the following two cases. 1) If two signed graphs are identified as isomorphic by random walk encoding, two graphs are also identified as
isomorphic by the Extended WL-test For Signed Graph. 2) Two different signed graphs can be distinguished as non-isomorphic by random walk encoding but cannot by the Extended WL-test For Signed Graph.

In case 1, two signed graphs $G_1$, $G_2$ are identified isomorphic by random walk encoding with sufficient walk length such that all 1-hop positive and negative neighbors of each node in $G$ can be explored. Then we can obtain the signed distance between every pair of nodes. Let multiset $\mathcal{S}^1=\{s_1^1,...,s_{|\mathcal{V}_1|}^1\}$ be the random walk distances between $v_{i}$ and the other nodes in $G_1$, there exists corresponding multiset $\mathcal{S}^2=\{s_1^2,...,s_{|\mathcal{V}_2|}^2\}$  for $v_{j}$ in $G_2$. The number of $s_{i}^1=1$ and $s_{j}^2=1$, the number of $s_{i}^1=-1$ and $s_{j}^2=-1$ in the multisets must be the same. Therefore, the Extended WL-test For Signed Graph also views the two graphs as isomorphic.

As illustrated in case 2 (Figure \ref{fig:enter-label}), the structures of the two graphs are different. However, the Extended WL-test For Signed Graph cannot identify these two graphs as non-isomorphic. In the left graph, node $v_1$ and node $v_4$ have one positive neighbor and two negative neigbors and the remaining nodes all have one negative neighbor and two positive neighbors. In the right graph, node $v_1$ and node $v_4$ also have one positive neighbor and two negative neigbors and the remaining nodes also have one negative neighbor and two positive neighbors. Even if each node aggregates the neighbor node information within two hops, the number of positive and negative neighbors of the nodes in the two graphs corresponds one-to-one. For example, node $v_1$ has one positive neighbor and four negative neighbors within two hops in the left graph and node $v_1$ also has one positive neighbor and four negative neighbors within two hops in the right graph. Therefore, the Extended WL-test For Signed Graph views them as isomorphic. In contrast, random walk encoding regards two graphs as non-isomorphic. For example, one can revisit the node in 3 steps in the left graph, while it is impossible on the right graph. So random walk encoding is strictly more powerful than the Extended WL-test For Signed Graph. The theorem follows.
\end{proof}
Theorem \ref{the:extended} shows that signed random walk encoding is more powerful than the Extended WL-test For Signed Graph, which illustrates that our graph Transformer with signed random walk encoding is more powerful than SGCN.
Then we can conclude the following theorem:
\renewcommand{\thetheorem}{2}
\begin{theorem}
With a sufficient number of random walk length, the graph transformer architecture based on signed random walk encoding is more expressive than SGCN.
\end{theorem}
The theorem \ref{the:2} follows.

\subsection{Signed Random Walk with Restart (SRWR)}
The SRWR model can provide personalized trust or distrust rankings reflecting signed edges based on balance theory. It introduces a 
signed random surfer for bipartite signed graph. The surfer randomly surfs
between nodes and traverses their relationships. Initially, the surfer carries a positive sign $+$ at node s. Suppose the surfer is currently at node u and c is the restart probability of the surfer. The surfer can randomly moves to one of the neighbors from node u with probability 1$-$c or restart with probability c.  When the random surfer encounters a negative edge, she changes her sign from positive to negative, or vice versa. Otherwise, she keeps her sign.
$\mathbf{r}_u^+$ is the probability that the positive surfer
is at node u after SRWR from the seed node s and $\mathbf{r}_u^-$ is the probability that the negative surfer is at node u after SRWR from the seed node s.
$\mathbf{r}_u^+$ and $\mathbf{r}_u^-$ are defined similarly as follows:
\begin{equation}
    \begin{aligned}
&\mathbf{r}_{u}^{+} =(1-c)\left(\sum_{v\in{\overleftarrow{\mathbf{N}}}_{u}^{+}}\frac{\mathbf{r}_{v}^{+}}{|{\overrightarrow{\mathbf{N}}}_{v}|}+\sum_{v\in{\overleftarrow{\mathbf{N}}}_{u}^{-}}\frac{\mathbf{r}_{v}^{-}}{|{\overrightarrow{\mathbf{N}}}_{v}|}\right)+c\mathbf{1}(u=s) \\
&\mathbf{r}_{u}^{-} =(1-c)\left(\sum_{v\in\mathbf{\overleftarrow{N}}_{u}^{-}}\frac{\mathbf{r}_{v}^{+}}{|\overrightarrow{\mathbf{N}}_{v}|}+\sum_{v\in\mathbf{\overleftarrow{N}}_{u}^{+}}\frac{\mathbf{r}_{v}^{-}}{|\overrightarrow{\mathbf{N}}_{v}|}\right) 
\end{aligned}
\end{equation}
where $\overrightarrow{\mathbf{N}}_i$ is the set of out-neighbors of node i,  $\overleftarrow{\mathbf{N}}_i$ is the set of in-neighbors of node i and $c\mathbf{1}(u=s)$ is the restarting score. 
Let $|A|$ be the absolute adjacency matrix of A, and D is the out-degree
diagonal matrix of $|A|$. Then semirow normalized matrix of A is $\tilde{\mathbf{A}}=\mathbf{D}^{-1}\mathbf{A}$. The positive semi-row normalized matrix $\tilde{\mathbf{A}}_+$ contains only positive values in the semi-row normalized matrix $\tilde{\mathbf{A}}$. The negative semi-row normalized matrix $\tilde{\mathbf{A}}_-$ contains absolute values of negative elements in $\tilde{\mathbf{A}}$. In other words, $\tilde{\mathbf{A}}=\tilde{\mathbf{A}}_+-\tilde{\mathbf{A}}_-$. Then Equation (1) is represented as the following equation:
\begin{equation}
    \begin{aligned}&\mathbf{r}^{+}=(1-c)\left(\mathbf{\tilde{A}}_{+}^{\top}\mathbf{r}^{+}+\mathbf{\tilde{A}}_{-}^{\top}\mathbf{r}^{-}\right)+c\mathbf{q}\\&\mathbf{r}^{-}=(1-c)\left(\mathbf{\tilde{A}}_{-}^{\top}\mathbf{r}^{+}+\mathbf{\tilde{A}}_{+}^{\top}\mathbf{r}^{-}\right)\end{aligned}
\end{equation}
where $\mathbf{q}$ is a vector whose s-th element is 1, and all other
elements are 0.
However, balance theory cannot completely explain the relationships in real-world signed graphs, where situations like "a friend of a friend is an enemy" can also occur. To reflect the uncertainty of the friendship of an enemy, the ranking model introduces balance attenuation factors $\beta$ and $\gamma$. $\beta$ is a parameter for the uncertainty of ”the enemy of my enemy is my friend”, and $\gamma$ is for ”the friend of my enemy is my enemy”. So The Equation(2) is further described as: 
\begin{equation}
\begin{aligned}
    \mathbf{r}^{+} &= (1-c)\left(\mathbf{\tilde{A}}_{+}^{\top}\mathbf{r}^{+}+\beta\mathbf{\tilde{A}}_{-}^{\top}\mathbf{r}^{-}+(1-\gamma)\mathbf{\tilde{A}}_{+}^{\top}\mathbf{r}^{-}\right) + c\mathbf{q}, \\
    \mathbf{r}^{-} &= (1-c)\left(\tilde{\mathbf{A}}_{-}^{\top}\mathbf{r}^{+}+\gamma\tilde{\mathbf{A}}_{+}^{\top}\mathbf{r}^{-}+(1-\beta)\tilde{\mathbf{A}}_{-}^{\top}\mathbf{r}^{-}\right)
\end{aligned}
\end{equation}
Then $\mathbf{r}^d=\mathbf{r}^+-\mathbf{r}^-$ is considered as a relative trustworthiness vector of nodes, which is used as a personalized ranking.

For every node, the SRWR alogorithm can return a $\mathbf{r}^{+}$ and a $\mathbf{r}^{-}$. We combine all $\mathbf{r}^{+}$ and $\mathbf{r}^{-}$ into a positive matrix $\mathbf{r}_{\mathbf{p}}\in\mathbb{R}^{n\times n}$ and a negative matrix $\mathbf{r}_{\mathbf{n}}\in\mathbb{R}^{n\times n}$ respectively. However, this method has some issues. The calculated $\mathbf{r}_{\mathbf{p}}(u,v)$(or $\mathbf{r}_{\mathbf{n}}(u,v)$) may differ from the calculated $\mathbf{r}_{\mathbf{p}}(v,u)$(or $\mathbf{r}_{\mathbf{n}}(v,u)$). To solve the question, for $\mathbf{r}_{\mathbf{p}}$, we transpose $\mathbf{r}_{\mathbf{p}}$ to $\mathbf{r}_{\mathbf{p}}T$ where $\mathbf{r_{p}}^{T}(u,v) = \mathbf{r_{p}}(v,u)$. 
We take the maximum value between $\mathbf{r_{p}}^{T}(u,v)$ and $\mathbf{r_{p}}(u,v)$ as $\mathbf{r}_{\mathbf{P}_{\mathrm{max}}}$ and $\mathbf{r}_{\mathbf{n}}$ is the same as $\mathbf{r}_{\mathbf{p}}$. Then we calculate $\mathbf{r}_{\mathbf{p}_{\max}}-\mathbf{r}_{\mathbf{n}_{\max}}$ for each pair of nodes to form a symmetric matrix $\mathbf{r}_{\mathbf{d}}$. Finally, we set thresholds $p$ and $n$ to evaluate the positive and negative scores for each pair of nodes in this matrix. A pair of nodes is considered a positive relationship if the score is positive and greater than or equal to $p$, and a negative relationship if the score is negative and less than or equal to $n$.
After performing these diffusion operations, we obtain a diffusion matrix denoted as $S $, which adds the positive and negative relationships between nodes and their multi-hop neighbors to A. $S_{ij} = 1$ indicates a positive edge between node $i$ and node $j$, while $ S_{ij} = -1$ indicates a negative edge between them and $S_{ij}=0 $ signifies no relationship between the two nodes.

\subsection{Loss Function}
\begin{align}\label{eq:loss}
&\mathcal{L}(\theta^{W},\theta^{MLG}) = \nonumber \\
&- \frac{1}{\mathcal{M}}  \sum\limits_{(u_i,u_j,s) \in \mathcal{M}} \omega_s \log \frac{\exp{([{\bf z}_i,{\bf z}_j]\theta^{MLG}_s)}}{\sum\limits_{q \in \{+,-,?\}} \exp{([{\bf z}_i,{\bf z}_j]\theta^{MLG}_q )}} \nonumber \\
&+ \lambda \Bigg[ \frac{1}{|\mathcal{M}_{(+,?)}|} \sum\limits_{\substack{(u_i,u_j,u_k) \\ \in \mathcal{M}_{(+,?)}}} \max\Big(0,(|| {\bf z}_i - {\bf z}_j||^2_2 - || {\bf z}_i - {\bf z}_k||^2_2) \Big) \nonumber \\
&  + \frac{1}{|\mathcal{M}_{(-,?)}|} \sum\limits_{\substack{(u_i,u_j,u_k) \\ \in \mathcal{M}_{(-,?)}}} \max\Big(0,(|| {\bf z}_i - {\bf z}_k||^2_2 - || {\bf z}_i - {\bf z}_j||^2_2) \Big) \Bigg] \nonumber \\
&+ Reg(\theta^{W},\theta^{MLG})
\end{align}

\subsection{Datasets}
To evaluate the link sign prediction accuracy and interpretability of our model, we conducted extensive experiments on real-world datasets.
In our experiments, we utilized a diverse set of real-world datasets that include both positive and negative feedback. These datasets are essential for studying various aspects of user behavior, trust networks, and community interactions. Below is a detailed description of each dataset used:

\textbf{Amazon-Music and Epinions} These widely used datasets contain user ratings on items from the Amazon and Epinions platforms. High ratings (greater than 3.5) are considered positive feedback, while lower ratings are deemed negative. Epinions, a popular product review website, allows users to categorize others into trusted and distrusted, represented as positive and negative links, respectively.

\textbf{KuaiRec and KuaiRand} These datasets record user behavior within the KuaiApp. KuaiRec focuses on a dense dataset where feedback classification is based on the ratio of user viewing duration to total video duration. Ratios equal to or exceeding 4 are positive, while those below 0.1 are negative. KuaiRand uses the “is\_click” attribute to classify feedback, as suggested by previous studies. Both datasets are split into training and testing sets in a ratio of 8:2.

\textbf{Bitcoin-Alpha and Bitcoin-OTC} These datasets come from sites where users buy and sell using Bitcoins. Users have created trust networks to rate others they trust (positive) or distrust (negative) to mitigate scams. These datasets provide insights into anonymous user interactions and trust dynamics.

\textbf{WikiElec and WikiRfa} These datasets include historical data on Wikipedia administrator elections and requests for adminship (RFA), respectively. They provide insights into community decision-making processes and user interactions, valuable for tasks such as symbolic graph prediction and sentiment analysis.

\begin{table}[h]
\centering
\begin{tabular}{lcccc}
\toprule
Dataset & \# Nodes & \# Pos & \# Neg & Pos/Neg \\ 
\midrule
WikiElec & 7,118 & 81,345 & 22,344 & 1:0.27 \\ 
KuaiRec & 3,327 & 40,216 & 239,165 & 1:5.95 \\ 
WikiRfa & 11,017 & 133,330 & 37,005 & 1:0.28 \\ 
Epinions & 17,894 & 301,378 & 112,396 & 1:0.37 \\ 
KuaiRand & 16,974 & 128,346 & 161,063 & 1:1.26 \\ 
Bitcoin-OTC & 6,006 & 32,028 & 3,563 & 1:0.11 \\ 
Amazon-Music & 3,472 & 40,043 & 9,832 & 1:0.25 \\ 
Bitcoin-Alpha & 7,605 & 22,649 & 1,536 & 1:0.07 \\ 
\bottomrule
\end{tabular}
\caption{Summary of datasets with nodes, positive and negative edges, and their ratio.}
\label{tab:dataset_summary}
\end{table}

\begin{figure}
    \centering
    \includegraphics[width=1\linewidth]{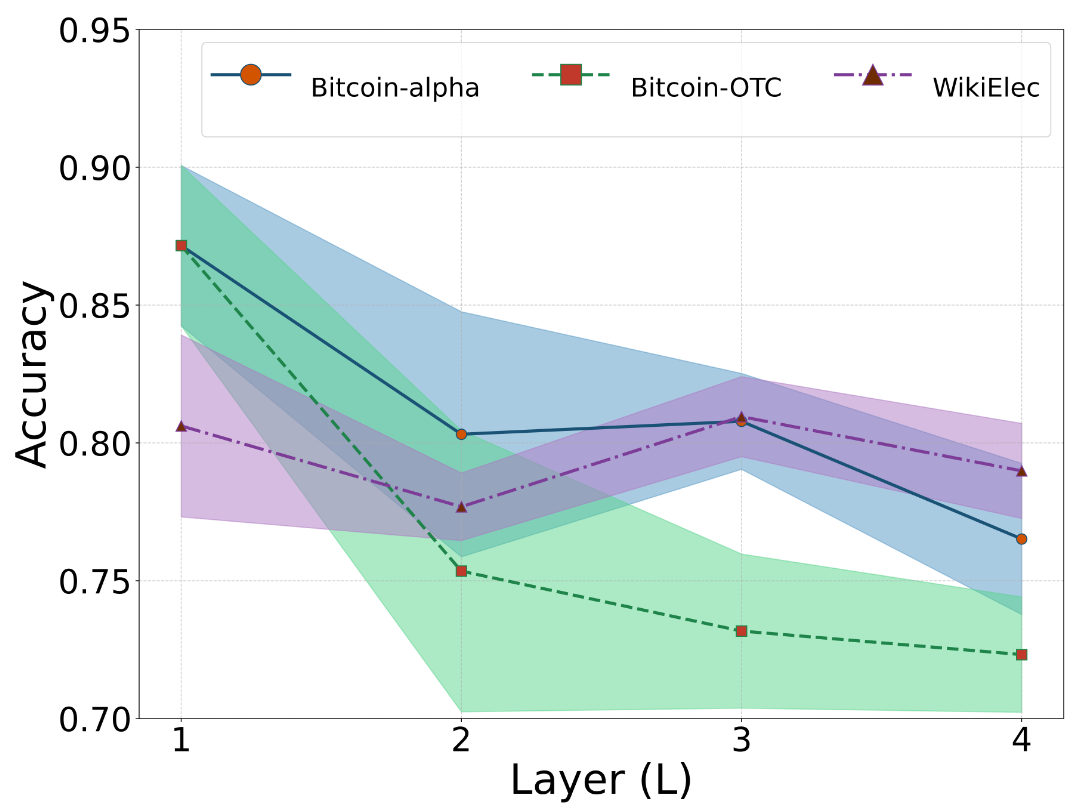}
    \caption{The sensitivity of parameter $L$}
    \label{fig:layer}
\end{figure}

\textbf{Explanation Information}. we enhance existing real-world datasets by adding explanation information. The construction process is as follows: we first load node embeddings and edge information. For each edge, we compute the distance between connected nodes and determine the k-nearest and k-farthest neighbors for each node. Then, we generate ground\_pos\_neighbors and ground\_neg\_neighbors, representing the sets of positive and negative neighbors, respectively. Specifically, for each node, we calculate the Euclidean distance to all other nodes and sort them by distance. For the positive neighbor set, we select the k nearest neighbors, and for the negative neighbor set, we select the k farthest neighbors. These neighbor sets are used as ground truth explanations to interpret the predictions of edges. By comparing the predicted neighbor sets with the ground truth neighbor sets, we can measure the accuracy of the explanations.
Using this approach, we can get datasets with explanation information. This dataset provides a controlled environment for quantitatively evaluating the effectiveness of different explanation methods.

\subsection{Baselines}
We compare the proposed framework with representative and state-of-the-art methods for link prediction, which include:

\textbf{GCN} \cite{kipf2016semi} is an initial and representative GNN model designed for unsigned graphs which employs an efficient layer-wise propagation rule.

\textbf{GAT} \cite{velivckovic2017graph} adopts an attention-based architecture which can learn different weights to neighbors. It is also designed for unsigned graphs.

\textbf{SGCN} \cite{derr2018signed} generalizes GCN to signed graphs by designing a new information aggregator which is based on balance theory.

\textbf{SNEA} \cite{li2020learning} generalizes GAT to signed graphs which adopts attention-based aggregators in the message passing mechanism and is also based on the balance theory.

\textbf{SGCL} \cite{shu2021sgcl} generalizes graph contrastive learning to signed graphs, which employs graph augmentations to reduce the harm of interaction noise and enhance the model robustness.

\textbf{SIGformer} \cite{chen2024sigformer} is an advanced method that utilizes the Transformer architecture for sign-aware graph-based recommendation. It introduces two novel positional encodings that effectively capture the spectral characteristics and path patterns of signed graphs, allowing for comprehensive utilization of the entire graph.

\subsection{Sensitivity analysis of parameter $L$} In Figure \ref{fig:layer} $L$ is varied from 1 to 4, revealing that increasing the number of layers generally reduces the model's mean performance, particularly for the BitcoinOTC dataset, suggesting that deeper models may lead to overfitting or added complexity.

\end{document}